%% file: main.tex
\documentclass[11pt]{article}

\input{header}

\title{\bf{Regret Bound Balancing and Elimination for Model Selection in Bandits and RL}}
\author{%
   Aldo Pacchiano\\
   University of California, Berkeley\\
   {\small\texttt{pacchiano@berkeley.edu}}\\
  \and
  Christoph Dann\\
  Google Research\\
  {\small\texttt{cdann@cdann.net}}\\
  \and
  Claudio Gentile\\
  Google Research\\
  {\small\texttt{cgentile@google.com}}\\
  \and
  Peter Bartlett\\
  University of California, Berkeley\\
  {\small\texttt{peter@berkeley.edu}}
}

\begin{document}

\maketitle

\begin{abstract}
We propose a simple model selection approach for algorithms in stochastic bandit and reinforcement learning problems. As opposed to prior work that (implicitly) assumes knowledge of the optimal regret, we only require that each base algorithm comes with a candidate regret bound that may or may not hold during all rounds. In each round, our approach plays a base algorithm to keep the candidate regret bounds of all remaining base algorithms balanced, and eliminates algorithms that violate their candidate bound. We prove that the total regret of this approach is bounded by the best valid candidate regret bound times a multiplicative factor. This factor is reasonably small in several applications, including linear bandits and MDPs with nested function classes, linear bandits with unknown misspecification, and LinUCB applied to linear bandits with different confidence parameters. We further show that, under a suitable gap-assumption, this factor only scales with the number of base algorithms and not their complexity when the number of rounds is large enough. 
Finally, unlike recent efforts in model selection for linear stochastic bandits, our approach is versatile enough to also cover cases where the context information is generated by an adversarial environment, rather than a stochastic one.
\end{abstract}

\newpage
\tableofcontents
\addtocontents{toc}{\protect\setcounter{tocdepth}{2}}
\clearpage

\input{intro}

\input{related_work}
\input{prelim}

\input{stoch_ctx_algorithm}

\input{applications}

\input{adversial_ctx}

\input{conclusion}

\bibliographystyle{abbrvnat}
\bibliography{manual}

\onecolumn
 \appendix
\input{appendix/stochastic_proofs}
\input{appendix/adversarial_proofs}

\input{appendix/misc_appendix}

\end{document}

%% file: header.tex
\usepackage{fullpage}
\usepackage{bbm}
\usepackage{mathtools}
\usepackage{natbib}
\usepackage{verbatim}
\usepackage{comment}
\usepackage[linesnumbered, ruled,vlined]{algorithm2e}

\usepackage{wrapfig}

\usepackage{microtype}
\usepackage{graphicx}
\usepackage{subfigure}
\usepackage{booktabs} %

\usepackage{collectbox}

\usepackage{hyperref}

\usepackage{tikz}
\usetikzlibrary{intersections,shapes.arrows}

\usepackage{amsthm,amsmath,amssymb, amsfonts}
\usepackage{thmtools}

\newtheorem{theorem}{Theorem}[section]

\newtheorem{lemma}[theorem]{Lemma}

\newtheorem{definition}[theorem]{Definition}
\newtheorem{remark}[theorem]{Remark}

\usepackage{comment}
\usepackage{xcolor}
\usepackage{mathrsfs}
\usepackage{enumitem}
\usepackage{bbm}
\usepackage{mdframed}
\usepackage{geometry}

\usepackage{pythonhighlight}

\usepackage{prettyref}
\newcommand{\pref}[1]{\prettyref{#1}}

\newcommand{\savehyperref}[2]{\texorpdfstring{\hyperref[#1]{#2}}{#2}}
\newrefformat{eq}{\savehyperref{#1}{\textup{(\ref*{#1})}}}
\newrefformat{eqn}{\savehyperref{#1}{Equation~\ref*{#1}}}
\newrefformat{con}{\savehyperref{#1}{Conjecture~\ref*{#1}}}
\newrefformat{lem}{\savehyperref{#1}{Lemma~\ref*{#1}}}
\newrefformat{def}{\savehyperref{#1}{Definition~\ref*{#1}}}
\newrefformat{line}{\savehyperref{#1}{line~\ref*{#1}}}
\newrefformat{thm}{\savehyperref{#1}{Theorem~\ref*{#1}}}
\newrefformat{corr}{\savehyperref{#1}{Corollary~\ref*{#1}}}
\newrefformat{sec}{\savehyperref{#1}{Section~\ref*{#1}}}
\newrefformat{app}{\savehyperref{#1}{Appendix~\ref*{#1}}}
\newrefformat{ass}{\savehyperref{#1}{Assumption~\ref*{#1}}}
\newrefformat{ex}{\savehyperref{#1}{Example~\ref*{#1}}}
\newrefformat{fig}{\savehyperref{#1}{Figure~\ref*{#1}}}
\newrefformat{tab}{\savehyperref{#1}{Table~\ref*{#1}}}
\newrefformat{alg}{\savehyperref{#1}{Algorithm~\ref*{#1}}}
\newrefformat{rem}{\savehyperref{#1}{Remark~\ref*{#1}}}
\newrefformat{conj}{\savehyperref{#1}{Conjecture~\ref*{#1}}}
\newrefformat{prop}{\savehyperref{#1}{Proposition~\ref*{#1}}}
\newrefformat{proto}{\savehyperref{#1}{Protocol~\ref*{#1}}}
\newrefformat{prob}{\savehyperref{#1}{Problem~\ref*{#1}}}
\newrefformat{claim}{\savehyperref{#1}{Claim~\ref*{#1}}}

\usepackage{cdann_definitions}

\newmdtheoremenv{test}{Test}

\definecolor{DarkRed}{rgb}{0.75,0,0}
\definecolor{DarkGreen}{rgb}{0,0.5,0}
\definecolor{DarkPurple}{rgb}{0.5,0,0.5}
\definecolor{DarkBlue}{rgb}{0,0,0.7}

\hypersetup{colorlinks=true, plainpages = false, citecolor = DarkBlue, urlcolor = DarkBlue, filecolor = black, linkcolor =DarkBlue}

\newcommand{\regret}{\mathsf{Reg}}
\newcommand{\regretbound}{R}
\newcommand{\misspidx}{\Bcal}
\newcommand{\wellspidx}{\Wcal}
\newcommand{\bestallwell}{*}

\newcommand{\oful}{\textsc{OFUL}}

%% file: intro.tex
\section{Introduction}

Multi-armed bandits are a general framework of sequential decision making that has in the last two decades received a lot of attention. The main aspect of this framework is a sequence of $T$ {\em rounds} of interaction between a learning agent and an unknown environment. During each round, the learner picks an action from a set of available actions on that round, and the environment consequently generates a feedback (e.g., in the form of a {\em reward} value) associated with the chosen action.
Given a class of benchmark policies, the goal of the learning agent is to accumulate during the course of the $T$ rounds a total reward which is not much smaller than that of the best policy in hindsight within the benchmark class.

Multi-armed bandits have found applications in a wide variety of domains, like clinical trials (e.g., \cite{vbw15}), online advertising (e.g., \cite{sbf17}), recommendation systems (e.g., \cite{li+10}), and beyond. 

Since many bandit methods are often deployed at scale in industrial applications, the complexity and diversity of the involved learning solutions typically require being able to select among several alternatives, like selecting the best within a pool of algorithms, or even alternative configurations of the same algorithm (as in, e.g., hyperpararameter optimization). Hence, the problem of {\em model selection} in  bandit algorithms has become chiefly important in order to simplify the development of data processing pipelines at scale while simultaneously achieving improved statistical performance.

In this paper, we study the problem of online model selection among a set of alternative learning algorithms, these algorithms being themselves bandit algorithms. Each such algorithm %
is designed to work well only when favorable conditions are satisfied. Yet, the algorithm designer may not know in advance which one of them is more appropriate for the problem at hand.

As a simple example, many known multi-armed bandit algorithms, such as UCB (e.g., \cite[Ch. 7]{lattimore2018bandit}), rely on a confidence interval width as prescribed by a theoretical recipe. However, it has been observed multiple times in practice that setting this width smaller than theoretically suggested can lead to substantial performance improvements. On the other hand, picking too small a width can lead to a dramatic degradation in performance that may translate into a linear regret. It is therefore desirable to design theoretically sound model selection procedures that can help us find an optimal parameter setting in an online fashion.

Another simple example comes from trying to distinguish between a contextual and a non-contextual environment. In e-commerce problems, even if contextual information is available about users and the transaction at hand, it may prove more beneficial to use a simple UCB style algorithm that ignores the context or that only uses part of the context information. A model selection strategy that selects when or to what extent making use of contextual information can lead to better performance for contextual bandit algorithms.

%% file: related_work.tex
\section{Related Work and our Contribution}
In this paper we aim to develop a general purpose model selection master algorithm (that is, aggregation approach) that can be combined with multiple base bandit algorithms, and is able to obtain regret guarantees competitive with respect to the best base algorithm. 

The problem of online model selection for bandit algorithms has received a lot of recent attention, as witnessed by a flurry of recent works (e.g., \cite{agarwal2017corralling,foster2019model,chatterji2020osom,pacchiano2020model,arora2020corralling,abbasi2020regret,foster2020adapting,lee2020online,bibaut2020rate,ghosh2020problem}).

These previous works on model selection can be broadly split into two approaches: (i) Approaches that make use of an adversarial master algorithm, and (ii) approaches that rely on a statistical test which is able to detect when a base algorithm is misspecified. Our approach, called {\em Regret Balancing and Elimination}, falls squarely in the second camp. 

Within the first category are the so-called {\em corraling} algorithms. These yield statistical guarantees of the form $\mathcal{O}(d_\star^{\alpha} T^{\beta})$ for arbitrary $\alpha \geq 1, \beta < 1$, where $d_\star$ depends generally on the complexity of the best model class or algorithm and other problem parameters. The original Corraling Algorithm of \cite{agarwal2017corralling} relies on an adversarial master algorithm based on mirror descent that can be combined with many base algorithms (both stochastic and adversarial), provided these base algorithms satisfy a stability guarantee. In this case, the base algorithms are fed with an importance-weighted estimator of the reward, hence they have to be robust to potentially wide fluctuations in the reward scaling, due to the evolving nature of the master algorithm's distribution over base algorithms. Unfortunately, in order to show that a base algorithm can be combined with the corralling master to satisfy a valid model selection regret guarantee, it is necessary to verify that the above-mentioned stability condition holds, something that has to be done on a case-by-case basis. 
The model selection guarantee is of the form $\mathcal{O}\left( \sqrt{MT} + MR_{i_\star}(T)  \right)$, where $M$ is the number of base algorithms and $R_{i_\star}(T)$ is the regret guarantee of any of the base algorithms. Yet, this is achieved only if the master's learning rate is set as a function of $R_{i_\star}(T)$, a quantity which is typically unknown. 

Some of the shortcomings of the original Corralling Algorithm have been addressed by the more recent work of ~\cite{pacchiano2020model}. The authors propose a generic model selection procedure to combine stochastic bandit algorithms with an adversarial master. As opposed to the corralling algorithm of~\cite{agarwal2017corralling}, the {\em Stochastic Corral} method in \citep{pacchiano2020model} allows the use of any stochastic bandit algorithm in stochastic contextual environments (the contexts are i.i.d.), provided it satisfies a high probability regret guarantee, thus relaxing the stability condition in \citep{agarwal2017corralling}. \cite{pacchiano2020model} obtain the following model selection guarantees: When the base algorithms have a regret bound of the form $\{d_i T^{\alpha} \}_{i=1}^M$, Stochastic Corral achieves a regret guarantee of $\widetilde{\mathcal{O}}(  \sqrt{MT} + M^\alpha T^{1-\alpha} + M^{1-\alpha} T^\alpha d_{i_\star}^{1/\alpha} )$ when using a Corralling Algorithm as master, and a rate of $\widetilde{\mathcal{O}}(  \sqrt{MT} + M^{\frac{1-\alpha}{2-\alpha} } T^{\frac{1}{2-\alpha}}d_{i_\star})$ under a forced exploration EXP3 (e.g., \cite[Ch. 11]{lattimore2018bandit}) master. Despite these advances, it remains unclear how to avoid the $\sqrt{MT}$ cost of a corralling approach. Our approach recovers and improves on the guarantees obtained by \cite{agarwal2017corralling} and \cite{pacchiano2020model} in two ways. First, we propose a general purpose stochastic master algorithm that can be used in combination with any set of stochastic bandit algorithms. As opposed to the adversarial master algorithms of \cite{agarwal2017corralling} and \cite{pacchiano2020model}, ours is much more interpretable and transparent. Second, due to the stochastic nature of our master algorithm, we are able to prove {\em gap-dependent} bounds, thereby departing from the inherent $\sqrt{T}$ limit of adversarial master approaches. Furthermore, the memory requirements of \cite{pacchiano2020model} are very onerous, since their algorithm requires to store all the policies played by the base algorithms. Our algorithm's memory requirements are minimal in comparison.

There exist other related approaches in the literature that make use of an adversarial corralling master algorithm as a means of performing model selection. \citet{arora2020corralling} propose an approach based on a Tsallis-INF adversarial master, which is able to recover gap-dependent regret guarantees for stochastic bandit problems. Nevertheless, their approach suffers from the drawback that whenever the rates of the input base algorithms are of the form $\{d_i T^{\alpha} \}_{i=1}^M$, where $d_1 \leq \cdots \leq d_M$, they obtain a regret guarantee for their master algorithm of the form $d_{M}T^{\alpha}$, a quantity that could be substantially worse than the regret achievable by the optimal base algorithm $d_{i_\star} T^{\alpha}$, since $d_{i_\star}$ might be much smaller than $d_M$. In contrast, our approach achieves a rate of $d_{i_\star}^2 T^{\alpha}$. Other related approaches that make use of a Tsallis-INF adversarial master have also been proposed, e.g., \citet{foster2020adapting} achieve optimal rates for selecting the the misspecification level in the setting of contextual linear bandits.
In the setting of stochastic linear bandits with adversarial contexts, our approach can be seen to achieve the same model selection rates as~\citet{foster2020adapting} for the problem of selecting the best level of misspecification. 

As for the approaches that rely on a statistical test to perform model selection, %
minimax-optimal guarantees have been shown under strong eigenvalue assumptions on the context distribution by leveraging the special structure of the stochastic linear contextual bandit setting~\citep{foster2019model,chatterji2020osom}. These algorithms work by maintaining a set of active base learners, and playing a low complexity algorithm/model within the set.
If enough information is gathered to conclude that a higher complexity model better describes the observed data, they eliminate the low complexity model from the active set, and proceed to play a more complex one. 
Unlike those papers, we are able to get results for the nested linear class problem (initially studied by~\cite{foster2019model}), but without resorting to eigenvalue assumptions on the context distribution, and without relying on the finiteness of the action space.

In the more general task of selecting among different stochastic bandit algorithms operating in a stochastic environment (with i.i.d. contexts), the recent work~\citep{abbasi2020regret} has taken some steps towards proposing a stochastic master algorithm that can combine multiple stochastic base bandit algorithms, and obtain regret guarantees of the same nature or better than Stochastic Corral. \cite{abbasi2020regret} introduce an intriguing new technique for model selection referred to as Regret Balancing. At a high level, the main idea is to estimate the empirical regret of the base algorithms during the rounds that the algorithms are played, and ensure that all base algorithms suffer roughly the same empirical regret. As opposed to~\citep{foster2019model,chatterji2020osom} the Regret Balancing approach of~\cite{abbasi2020regret} does not eliminate any base algorithm.\footnote
{
Technically speaking the methods in \citep{foster2019model,chatterji2020osom} do not eliminate base algorithms, but reject a statistical hypothesis on the base algorithms' model complexity. 
} 
Unfortunately, in order for this approach to work, the exact scaling of the target optimal regret guarantee is required, which is again typically unknown. Our approach to model selection expands on the fundamental insights of regret balancing but, in contrast to~\citep{abbasi2011improved}, we are able to obtain results when model selecting among multiple base algorithms with different regret guarantees.

In~\cite{lee2020online} the authors propose ECE (Explore Commit Exploit), a model selection algorithm on stochastic contextual bandit algorithms. ECE can be thought of as an epsilon-greedy approach to the problem of model selection. Correspondingly, the regret guarantees of ECE have a dependence on $T$ of the order of $T^{2/3}$, in contrast to our typical $T^{1/2}$ dependence. A regret of the form $T^{2/3}$ is the same as the one achievable by a forced exploration EXP3 master in~\cite{pacchiano2020model}. \cite{lee2020online} also present gap-dependent guarantees under the same assumptions as in~\cite{arora2020corralling} (see also \cite{bibaut2020rate}): each algorithm satisfies a valid regret guarantee w.r.t. its own policy class. Our work does not rely on this restrictive assumption, in that we only require the optimal algorithm to be well behaved and satisfy its theoretical regret guarantee. This is because we admit the presence of regret-misspecified base algorithms in the pool, and compete against the best among the well-specified ones.
When the rates of the base algorithms are of the form $\{ d_i T^{\alpha}\}_{i=1}^M$ and in the regime where $T$ is much larger than $d_i$, our approach strictly dominates ECE's rates. Other works provide model selection results for specific bandit models, most notably,~\cite{ghosh2020problem} consider the problem of selecting over nested feature structures and an unknown parameter norm in the case of contextual linear bandits over a sphere. 
Our results recover model selection rates for these problems without requiring restrictive assumptions on the nature of the contexts. %

\subsection{Content of the paper} 
Building on \citet{abbasi2020regret}, we study a general regret bound balancing and elimination algorithm (\pref{sec:balancingandelimination}) for selection among a pool of base bandit algorithms, each coming with a {\em presumed} regret bound that may or may not hold. The master algorithm does not know a priori the identity of the base algorithms whose regret bounds hold. Under these general assumptions, we show that this master algorithm enjoys general regret guarantee (\pref{sec:regretanalysis}) that can be specialized to either the gap-independent (Section \ref{s:gapindependent}) or the gap-dependent (Section \ref{s:gapdependent}) case. Then, we specialize to relevant application examples with nested model classes (\pref{sec:applications}) that consider linear contextual bandits or linear Markov decision processes as base learners (\pref{sec:linbannested} and \ref{sec:linearmdpnested}). We also consider therein the unknown misspecification case (Section \ref{sec:unknownapx}), as well as the practically relevant problem of optimally tuning linear contextual bandit algorithms like \oful\, (\pref{sec:optoful}). Finally, we specifically focus on the nested linear contextual bandit setting of \pref{sec:linbannested}, and extend our balancing and elimination technique to the case where the context information is generated adversarially (Section \pref{sec:adv}). Despite we do not show this explicitly, similar extensions can be exhibited for the scenarios we consider in \pref{sec:linearmdpnested}, \ref{sec:unknownapx}, and \ref{sec:optoful}. 

In the next section, we introduce our basic setup and notation for stochastic contexts. For the adversarial context case, further elements of the setup with be given in \pref{sec:adv}. Most of our proofs are provided in the appendix.

%% file: prelim.tex
\section{Setup and Assumptions}\label{s:setup}
We consider contextual sequential decision making problems described by a context space $\Xcal$, an action space $\Acal$, and a policy space $\Pi = \{\pi\,:\,\Xcal \rightarrow \Acal\}$. At each round $t$, a context $x_t \in \Xcal$ is drawn\footnote
{
This assumption will actually be relaxed in \pref{sec:adv}.
} 
i.i.d. from some distribution, the learner observes this context, picks a policy $\pi_t \in \Pi$, thereby playing action $a_t = \pi_t(x_t) \in \Acal$, and receives an associated {\em reward} $r_t \in [0,1]$ drawn from some fixed distribution $\Dcal_{a_t, x_t}$ that may depend on the current action and context.

\paragraph{Base learners.}
Our learning policy in fact relies on base learner which are in turn learning algorithms operating in the same problem $\langle \Xcal, \Acal, \Pi\rangle$. Specifically, there are $M$ base learners which we index by $i \in [M] = \{1,\ldots, M\}$. In each round $t$, we select one of the base learners to play, and receive the reward associated with the action played by the policy deployed by that base learner in that round. Let us denote by $T_i(t) \subseteq \NN$ the set of rounds in which learner $i$ was selected up to time $t \in \NN$. Then 
the pseudo-regret $\regret_i$ our algorithm incurs over rounds $k \in T_i(t)$ due to the selection of base learner $i$ is
\begin{align}
    \regret_{i}(t) = \sum_{k \in T_i(t)} \left(\max_{\pi' \in \Pi} \EE[ r_k | \pi'(x_k), x_k]  - \EE[ r_k | \pi_k(x_k), x_k]\right)~,
\end{align}
and the total pseudo-regret $\regret$ of our algorithm is then
$\regret(t) = \sum_{i=1}^M \regret_{i}(t)$.

\paragraph{Candidate regret bounds.} Each base learner $i$ comes with a \emph{candidate} regret (upper) bound
$\regretbound_i \colon \NN  \rightarrow \RR_+$, which is a function of the number of rounds this base learner has been played.
This bound is typically known a-priori to us, and can also be random as long as the current value of the bound is observable, that is, we assume $\regretbound_i(n_i(t))$ is observable for all $i \in [M]$ and $t \in \NN$, being $n_{i}(t) = |T_{i}(t)|$ the number of rounds learner $i$ was played after $t$ total rounds.
Without loss of generality, we shall assume each candidate regret bound is non-decreasing, and increases by at most $1$ from one play to the next, i.e., 
\begin{align}
    0 \leq \regretbound_i(n) - \regretbound_i(n-1) \leq 1~,
\end{align}
for all number of rounds $n \in \NN$ and base learner $i \in [M]$, with
$\regretbound_i(0) =0$.

\paragraph{Well- and misspecified learners.} We call learner $i$ \emph{well-specified} if $\regret_i(t) \leq \regretbound_i(n_i(t))$ for all $t \in [T]$, with high probability over the involved random variables (see later sections for more details and examples), and otherwise \emph{misspecified} (or {\em bad}).
A well-specified base learning $i$ is then one for which the candidate regret bound $R_i(\cdot)$ is a reliable upper bound on the actual regret of that learner.

For a given set of base learners and corresponding regret upper bounds,
we denote the set bad learners 
by $\misspidx \subseteq [M]$, and the set of well-specified ones by 
\begin{align*}
\wellspidx = \left\{ i \in [M] \colon \forall t \in [T]\,\, \regret_i(t) \leq \regretbound_i(n_i(t)) \right\} = [M] \setminus \misspidx~.
\end{align*}
Notice that sets $\wellspidx$ and $\misspidx$ are random sets. As a matter of fact, these sets do also depend on the time horizon $T$, but we leave this implicit in our notation. We assume in our regret-analysis that there is always a well-specified learner, that is $\wellspidx \neq \varnothing$. We will show that in the applications we consider, this happens with high probability. The index $i^\star \in \wellspidx$ (or just $\star$ in subscripts) will be used for  any well-specified learner. 

Consistent with the previous notation, we denote the total reward accumulated by base learner $i$ after a total of $t$ rounds as
\begin{align*}
    U_i(t) = \sum_{k \in T_i(t)} r_k~,
\end{align*}
and the total sum of rewards as
$U(t) = \sum_{i \in [M]} \sum_{k \in T_i(t)} r_k$.
The expected reward of the optimal policy at the context $x_t$ at round $t$ will be denoted by
\begin{align*}
    \mu^\star_t = \max_{\pi' \in \Pi} \EE[ r | \pi'(x_t), x_t]
\end{align*}
and, when contexts are stochastic,  the expectation of $\mu^\star_t$ over contexts simply as $\mu^\star = \EE_x \left[ \mu^\star_t \right]$ which is a fixed quantity and independent of the round $t$.

\paragraph{Problem statement.} Our goal is to perform model selection in this setting: We devise sequential decision making algorithms that have access to base learners as subroutines and are guaranteed to have regret that is comparable to the smallest regret bound among all well-specified base learners despite not  knowing a-priori which base learners that are.

%% file: stoch_ctx_algorithm.tex
\section{Regret Bound Balancing and Elimination}\label{sec:balancingandelimination}
Our main algorithm follows the basic principle of regret bound {\em balancing}. The algorithm chooses the base learner in each round so as to make all presumed regret bounds evaluated at the number of rounds that the respective base learner was played to be roughly equal. To see why this achieves good total regret, assume for now all base learners are well-specified, so that they all satisfy their presumed regret bounds. Then, because the regret accrued by each base learner is bounded by its presumed regret bound, and these regret bounds are approximately equal, the total regret our algorithm incurs is at most $M$ times worse than had we only played the algorithm with the best presumed regret bound:
\begin{align*}
    \regret(T) = \sum_{i=1}^M \regret_i(T) \leq \sum_{i=1}^M \regretbound_i(n_i(T)) \approx M \min_{i \in [M]} \regretbound_i(n_i(T)) \leq M \min_{i \in [M]} \regretbound_i(T)~.
\end{align*}
Yet, the above only works if all base learners are well specified, which may not be the case. Besides, if we know all such learners are well specified, we could simply single out at the beginning of the game the learner whose regret bound is lowest at time $T$, and select that learner from beginning to end.
Our task becomes more interesting in the presence of learners that may violate their presumed regret bound, when we do not know the identity of such learners. In this case, a reasonable goal for our policy would be to compete in the regret sense against the best well-specified base learner.

In order to handle this more involved situation, we pair the above regret bound balancing principle with a misspecification test to identify and eliminate misspecified base learners. 
This test compares the time-average rewards $U_i(t)/n_i(t)$ and $U_j(t)/n_j(t)$ achieved by two base learners $i$ and $j$, and relies on the following concentration argument.
While $U_i(t)$ is random and observable, the optimal average reward $\mu^\star$ is deterministic and unknown. We consider the event where, for each base learner $i$ and each round $t$, the difference between $U_i(t)/n_i(t)$ and $\mu^*$ is close to the corresponding regret:
\begin{align*}
    \Gcal = \Bigg\{ \forall i \in [M], \,\,\forall t \in \NN & \colon |n_i(t) \mu^\star - U_i(t) - \regret_{i}(t)| 
    \leq c\sqrt{n_i(t)\,\ln \frac{M\ln n_i(t)}{\delta}} \Bigg\}~.
\end{align*}
We show in \pref{lem:mustar_conc} in the appendix that for an appropriate absolute constant $c$, this event has probability $1 - \delta$. This holds because, for each fixed $t$,
$U_i(t)$ concentrates around $\sum_{k \in T_i(t)}  \EE[ r_k | \pi_k(x_k), x_k]$, while 
$\sum_{k \in T_i(t)} \max_{\pi' \in \Pi} \EE[ r_k | \pi'(x_k), x_k]$ concentrates around $n_i(t)\,\mu^\star$, since contexts $x_k$ are generated in an i.i.d. fashion.
Now, since the pseudo-regret $\regret_i$ cannot be negative by definition, the conditions defining $\Gcal$ yield a lower-bound on $\mu^\star$ based on the rewards of each learner $i$~:
\begin{align}
    \mu^\star \geq \frac{U_i(t)}{n_i(t)} - c\sqrt{\frac{\ln (M\ln n_i(t) / \delta)}{n_i(t)}}~. \label{eqn:lower_bound_mu_star}
\end{align}
When the provided regret bound $\regret_i(t) \leq R_i(n_i(t))$ for learner $i$ holds (that is, when $i$ is well specified), then $\Gcal$ also yields an upper-bound for $\mu^\star$:
\begin{align}
    \mu^\star \leq \frac{U_i(t)}{n_i(t)} + c\sqrt{\frac{\ln (M\ln n_i(t) / \delta)}{n_i(t)}}
    + \frac{\regretbound_i(n_i(t))}{n_i(t)}~. \label{equation:upper_bound_mu_star}
\end{align}
Thus, if at any round $t$ the upper bound for $\mu^\star$ from learner $i$ contradicts the lower-bound from any other learner $j$, 
\begin{align*}
    \frac{U_i(t)}{n_i(t)} &+ c\sqrt{\frac{\ln (M\ln n_i(t) / \delta)}{n_i(t)}}
    + \frac{\regretbound_i(n_i(t))}{n_i(t)}
     < \frac{U_j(t)}{n_j(t)} - c\sqrt{\frac{\ln (M\ln n_j(t) / \delta)}{n_j(t)}},
\end{align*}
then we conclude that the upper bound on $\mu^*$ provided by learner $i$ is false, thereby showing that $i$ is misspecified, %
and can safely be eliminated. Conversely, this also shows that no well-specified learner $i \in \wellspidx$ can be eliminated.
Combining the elimination criterion with regret bound balancing yields our main algorithm, whose pseudocode is presented as \pref{alg:balanced_elim_general}. The algorithm is an action elimination scheme that maintains over time a set $\Ical_t$ of active learners/actions at time $t$, and undergoes an elimination procedure as described above. The way base learner $i_t$ is selected at each round guarantees the regret bound equalization we alluded to at the beginning of this section.
\begin{algorithm}
\SetInd{0.5em}{0.3em}
    $\Ical_1 \gets [M]$\tcp*{set of active learners}
    
    $U_i(0) = n_i(0) = 0$ for all $i \in [M]$

\For{round $t=1, 2, \dots, T$}{
  Pick the base learner as 
    $i_t \in \argmin_{i \in \Ical_t} \regretbound_i(n_{i}(t-1))$
    
  Play learner $i_t$ and receive reward $r_t$\\
  Update base learner $i$ with $r_t$\\
  Update $n_i(\cdot)$ and $U_i(\cdot)$\,:\\
    - $U_{i_t}(t) \leftarrow U_{i_t}(t-1)+r_t $\\
    - $n_{i_t}(t) \leftarrow n_{i_t}(t-1)+1 $
    
    $\mathcal I_{t+1} \gets \mathcal I_{t}$
    
    \ForEach{active base learner $i \in \mathcal I_t$}{
    Test for misspecification by checking\\
    $\frac{U_i(t)}{n_i(t)} 
    + \frac{\regretbound_i(n_i(t))}{n_i(t)} + c\sqrt{\frac{\ln (M\ln n_i(t) / \delta)}{n_i(t)}} < \max_{j \in \Ical_t} \frac{U_j(t)}{n_j(t)} - c\sqrt{\frac{\ln (M \ln n_j(t) / \delta)}{n_j(t)}}$\\
    \If{\textrm{above condition is triggered}}{
    $\Ical_{t+1} \gets \Ical_{t+1} \setminus \{ i \}$
    }
    }
}
\caption{Regret Bound Balancing and Elimination Algorithm}
\label{alg:balanced_elim_general}
\end{algorithm}

\section{Regret Analysis}\label{sec:regretanalysis}
We first derive a general upper-bound on the regret of \pref{alg:balanced_elim_general} that depends on the ratios $\frac{n_i(t_i)}{n_{\star}(t_i)}$ of how often a learner $i$ has been played compared to the best base learner. We will later bound this quantity for specific forms of candidate regret bounds $\regretbound_i$ and provide simpler and more interpretable regret bounds.

\begin{theorem}\label{thm:general_regbound}
With probability at least $1 - \delta$, the total regret of Algorithm~\ref{alg:balanced_elim_general} is bounded for all rounds $T$ as follows:
\begin{align}\label{eqn:general_regret_bound}
    \regret(T) \leq &~
    \sum_{i=1}^M \regretbound_{\star}(n_{\star}(t_i))%
   + \sum_{i \in \misspidx} \frac{n_i(t_i)}{n_{\star}(t_i)} \regretbound_{\star}(n_{\star}(t_i)) + 2M
   \nonumber
    \\
    &+ 
    2c \sum_{i \in \misspidx}  \left(1 + \sqrt{\frac{n_i(t_i)}{n_{\star}(t_i)}} \right)
    \sqrt{n_i(t_i) \ln\frac{M \ln T}{\delta}}~,
\end{align}
where $t_i$ is the last round where learner $i$ passed the elimination test, $\star \in \Wcal$ is any well-specified learner, and $c$ is a universal positive constant.
\end{theorem}

In order to prove this statement, we first show that \pref{alg:balanced_elim_general} indeed keeps all candidate regret bounds approximately equal (\pref{lem:balancing}) and that the regret of any learner that has not been eliminated can be upper-bounded in terms of $\regretbound_{\star}(\cdot)$, the smallest regret upper bound among the well-specified learners (\pref{lem:regbound_nonelim}).

\begin{lemma}[Regret Bound Balancing]\label{lem:balancing}
In \pref{alg:balanced_elim_general}, the regret bounds of all active learners are balanced at all times, i.e.,
\begin{align*}
    \regretbound_i(n_i(t)) \leq \regretbound_j(n_j(t)) + 1
\end{align*}
for all $i,j \in \Ical_t$ and $t \in \NN \cup \{0\}$.
\end{lemma}
\begin{proof}
At $t=0$, the regret bound for all learners is $0$ and the statement holds. For the sake of contradiction, assume now the claim is violated for the first time in round $t$, i.e., there is a $i,j \in \Ical_{t}$ such that
$
\regretbound_i(n_i(t)) > \regretbound_j(n_j(t)) + 1$. Then $i,j \in \Ical_{t-1}$ and $i$ must have been played in round $t$.
Further, by assumption on the candidate regret bounds
\[
\regretbound_i(n_i(t-1)) \geq \regretbound_i(n_i(t)) - 1 > \regretbound_j(n_j(t)) = \regretbound_j(n_j(t - 1))~,
\]
where the strict inequality follows from the violated claim and the equality holds because $j$ was not played at time $t$.
The resulting inequality $\regretbound_i(n_i(t-1)) > \regretbound_j(n_j(t - 1))$ contradicts the claim that $i$ was played at round $t$.
\end{proof}

\begin{lemma}\label{lem:regbound_nonelim}
In Algorithm \ref{alg:balanced_elim_general}
For any active learner $i \in \Ical_{t+1}$ and well-specified learner $\star \in \wellspidx$, the regret of $i$ is bounded in event $\Gcal$ as
\begin{align}
    \regret_i(t) &\leq 1 + \left( \frac{n_i(t)}{n_{\star}(t)} + 1\right) \regretbound_{\star}(n_{\star}(t))  + 2c \left(1 + \sqrt{\frac{n_i(t)}{n_{\star}(t)}} \right)
    \sqrt{n_i(t) \ln \frac{M \ln t}{\delta}}~,
\end{align}
where $c$ is a universal constant.
\end{lemma}
\begin{proof}
If $i \in \Ical_{t+1}$ remains active, then it must have passed the misspecification test in round $t$ and satisfy, for all\footnote
{
Recall that, under $\Gcal$, any $\star \in \Wcal$ will remain active.
} 
$ \star \in \Wcal$,
\begin{align*}
    \frac{U_i(t)}{n_i(t)} &+ c\sqrt{\frac{\ln (M \ln n_i(t) / \delta)}{n_i(t)}}
    + \frac{\regretbound_i(n_i(t))}{n_i(t)} 
    \geq \frac{U_\star(t)}{n_{\star}(t)} - c\sqrt{\frac{\ln (M \ln n_\star(t) / \delta)}{n_\star(t)}}~.
\end{align*}
Subtracting $\mu^\star$ from both sides and rearranging terms gives
\begin{align*}
    \mu^\star - \frac{U_i(t)}{n_i(t)} &- c\sqrt{\frac{\ln (M \ln n_i(t) / \delta)}{n_i(t)}}
    - \frac{\regretbound_i(n_i(t))}{n_i(t)} 
    \leq \mu^\star - \frac{U_\star(t)}{n_\star(t)} + c\sqrt{\frac{\ln (M \ln n_\star(t) / \delta)}{n_\star(t)}}~.
\end{align*}
Applying the definition of $\Gcal$, we obtain an inequality in terms of pseudo-regrets:
\begin{align*}
    \frac{\regret_i(t)}{n_i(t)} &- 2c\sqrt{\frac{\ln (M \ln n_i(t) / \delta)}{n_i(t)}}
    - \frac{\regretbound_i(n_i(t))}{n_i(t)} 
    \leq\frac{\regret_\star(t)}{n_\star(t)} + 2c\sqrt{\frac{\ln (M \ln n_\star(t) / \delta)}{n_\star(t)}}~.
\end{align*}
Multiplying both terms by $n_i(t)$ and rearranging terms gives
\begin{align*}
        \regret_i(t) 
        \leq &~ 2c\sqrt{\ln \frac{ M\ln n_i(t)}{\delta} n_i(t)}
    + \regretbound_i(n_i(t))
      + \frac{n_i(t)}{n_\star(t)} \regret_\star(t) + 2c  \sqrt{\ln \frac{M \ln n_\star(t)}{\delta}
      }
    \sqrt{\frac{n_i(t)}{n_\star(t)}}~.
    \end{align*}
We now upper-bound the RHS by (i) replacing $\ln n_\star(t) \leq \ln t$ in the log-terms, (ii) using the fact that $\star \in \wellspidx$ is well-specified to replace the pseudo-regret $\regret_\star(\cdot)$ by $\regretbound_\star(\cdot)$, and (iii) use the balancing condition from \pref{lem:balancing} to replace $\regretbound_i(n_i(t))$ by $\regretbound_\star(n_\star(t)) + 1$. This yields
\begin{align*}
        \regret_i(t) 
        \leq &~ 1 + \left(1 + \frac{n_i(t)}{n_\star(t)}\right)  \regretbound_\star(n_\star(t))
        +
        2c\sqrt{n_i(t) \ln \frac{ M \ln t}{\delta}}\left(1 + \sqrt{\frac{n_i(t)}{n_\star(t)}}\right)~,
\end{align*}
which is the claimed bound.
\end{proof}

We are now ready to prove \pref{thm:general_regbound}.

\begin{proof}[Proof of Theorem~\ref{thm:general_regbound}]
Let $t_i$ be the last round where learner $i$ passed the elimination test. Then the total regret can be bounded as
\begin{align*}
    \regret(T) & = \sum_{i=1}^M \regret_i(T) \leq \sum_{i \in \wellspidx} \regretbound_i(n_i(T))  + \sum_{i \in \misspidx} \regret_i(t_i).
\end{align*}
Applying \pref{lem:regbound_nonelim} on $\regret_i(t_i)$ for all $i \in \misspidx$ 
and the balancing condition from \pref{lem:balancing} on the regret-bound for $i \in \wellspidx$ gives the desired bound. Finally, \pref{lem:mustar_conc} in \pref{app:stoch_proofs} shows that  event $\Gcal$ has probability at least $1 - \delta$.
\end{proof}

The general regret bound contained in \pref{thm:general_regbound} will be instantiated to more concrete cases for certain classes of candidate regret bounds. This will lead us to explicitly control the ratios $n_i(t_i) / n_\star(t_i)$. We do so in turn in Section \ref{s:gapindependent} and in Section \ref{s:gapdependent}.

\subsection{Gap-Independent Regret Bounds}\label{s:gapindependent}
The regret guarantees in this section hold whenever there is a well-specified learner. These guarantees are independent of how much misspecified learners violate their presumed regret bounds (``gap" of the learner). In the next section, we will show that tighter guarantees can be achieved in cases where the gap is large, that is, when misspecified learners exceed their presumed bounds by a significant amount.

\begin{table}    
    \centering
    \bgroup
\def\arraystretch{1.8}
    \begin{tabular}{@{\hspace{0cm}}l|ll@{\hspace{0cm}}}
    Presumed Bounds $\regretbound_i$ & Regret Guarantee of \pref{alg:balanced_elim_general} & Proof \\
    \hline
        $d_i C n^{1/3}$ & 
    $ (M + B^{2/3} d_\star^2)\textcolor{DarkGreen}{d_\star C T^{1/3}} + d_\star^{3/2} \sqrt{BT}$
    & \pref{thm:poly_reg}
    \\
    $d_i Cn^{2/3}$ &
    $(M + B^{1/3}\sqrt{d_\star}) \textcolor{DarkGreen}{d_\star C T^{2/3}}
    + d_\star^{3/4}\sqrt{BT}
    $
    & \pref{thm:poly_reg}
    \\
    $d_i C\sqrt{n}$ 
    & $ (M + \sqrt{B} d_\star) \textcolor{DarkGreen}{d_\star C\sqrt{T}}$
    & \pref{thm:poly_reg}
    \\
    \hline
        $d_i C\sqrt{n \ln\frac{n}{\delta}}$ 
    & $ (M + \sqrt{B}d_\star) \textcolor{DarkGreen}{d_\star C\sqrt{T\ln \frac{T}{\delta}}} $
    & \pref{thm:sqrt_reg} \\

    $\epsilon_i n + C \sqrt{n}$
    &
    $M (\textcolor{DarkGreen}{\epsilon_\bestallwell T + C \sqrt{T}}) + M C^2$ 
    & \pref{thm:linTmissp} \\
    \end{tabular}
    \egroup

    \caption{Summary of our gap-independent regret guarantees
    In all bounds but the one in the 4th line, log factors are omitted for readability.
    In green is the regret guarantee of the best well-specified learner.}
    \label{tab:summary_worst_case}
\end{table}

The first class of candidate regret bounds we consider is $T^{\beta}$ with $\beta \in (0,1]$. More concretely, each learner comes with a candidate regret bound of the form
\begin{align}
\label{eqn:poly_reg_candidates}
    \regretbound_i(n) = d_i C n^{\beta} \wedge n~,
\end{align}
where  $d_i \geq 1$ is some parameter and $C \geq 1$ is some term that does not depend on $n$ or $i$. Note that the minimum with $n$ is without loss of generality as any learner satisfies the regret bound $n$ by our assumption on rewards being in $[0,1]$. Consistent with our assumptions from Section \ref{s:setup}, this minimum ensures that the regret bound can increase by at most $1$ in each round.
For candidate regret bounds of this form, we can show the following regret bound:

\begin{restatable}{theorem}{polyregthm}\label{thm:poly_reg}
If \pref{alg:balanced_elim_general} is used with candidate regret bounds in Equation~\eqref{eqn:poly_reg_candidates}, then its total regret is bounded with probability at least $1- \delta$ for all $T$ as
\begin{align*}
        \regret(T) & \leq
       \left( M 
        + 2 B^{1 - \beta} d_\star^{\frac{1}{\beta}-1}\right) d_\star CT^{\beta}
        + 5 d_\star^{\frac{1}{2 \beta}} c \sqrt{B T \ln \frac{M\ln T}{\delta}} + 2M,
\end{align*}
where $\star \in \Wcal$ is any well-specified learner and $B = |\misspidx|$ is the number of misspecified learners.
\end{restatable}

The first three entries in \pref{tab:summary_worst_case} summarize this result in the relevant cases where $\beta = \frac{1}{3}, \frac{1}{2}$ and $\frac{2}{3}$. 
When $\beta \geq 1/2$, our regret bound can recover the best $T^\beta$ rate. In particular, the bound of \pref{thm:poly_reg} recovers the regret bound guarantee of the best well-specified learner up to a multiplicative factor of the form $M + B^{1 - \beta} d_\star^{\frac{1}{\beta}-1}$. On the other hand, when $\beta < 1/2$ our bound scales sub-optimally as $\sqrt{T}$. 
This is not surprising since the lower bound by \citet{pacchiano2020model} indicates a $\Omega(\sqrt{T})$ barrier for model-selection based on observed rewards without additional assumptions.

We further show in the appendix that this result can be generalized to the case where the candidate regret bounds scale with additional logarithmic factors in the number of rounds, e.g. $\sqrt{n \ln n}$ as opposed to just $\sqrt{n}$ -- see \pref{thm:sqrt_reg} in Appendix \ref{app:stoch_proofs}.

We defer the full proof of \pref{thm:poly_reg} to \pref{app:stoch_proofs}, but provide a brief sketch of the main argument for the special case of $\beta = \frac{1}{2}$. The general case follows analogously. 
\begin{proof}[Proof sketch of Theorem~\ref{thm:poly_reg}] 
The first term of the general regret bound from \pref{thm:general_regbound} can be written as $\sum_{i=1}^M \regretbound_{\star}(n_{\star}(t_i)) \leq M \regretbound_{\star}(T) \leq M C d_\star \sqrt{T}$, the first inequality using the monotonicity of the candidate regret bound. This yields the first term in \pref{thm:poly_reg}.
The second term in \pref{thm:general_regbound} can be controlled as follows:
\begin{align*}
    & \sum_{i \in \misspidx} \frac{n_i(t_i)}{n_{\star}(t_i)} \regretbound_{\star}(n_{\star}(t_i))
    \overset{(i)}{\leq} \sum_{i \in \misspidx} \frac{n_i(t_i)}{n_{\star}(t_i)} C d_\star \sqrt{n_\star(t_i)}
    = C d_\star\sum_{i \in \misspidx} \sqrt{\frac{n_i(t_i)}{n_{\star}(t_i)}}  \sqrt{n_i(t_i)}\\
    &\overset{(ii)}{\leq}
    C d_\star \sqrt{\sum_{i \in \misspidx} \frac{n_i(t_i)}{n_{\star}(t_i)}}  \sqrt{ \sum_{i \in \misspidx} n_i(t_i)}
    \overset{(iii)}{\leq} C d_\star \sqrt{2 B d_\star^2} \sqrt{t_i} \leq C d_\star^2 \sqrt{2 B T}~, 
\end{align*}
where step $(i)$ applies the definition of the candidate regret bound, step $(ii)$ uses Cauchy-Schwarz inequality and step $(iii)$ follows from the fact that the total number of plays at round $t_i$ is $t_i$ and a bound on the sum of play ratios 
$\sum_{i \in \misspidx} \frac{n_i(t_i)}{n_\star(t_i)}  \leq 2 B d_\star^2$, which we will show below. This yields the second term in the desired regret bound. The remaining terms can handled in a similar manner.

To derive the bound on the play ratios, consider first the case where $n_i(t_i)$ is so large that $\regretbound_i(n_i(t_i)) < n_i(t_i)$. Then, by the balancing condition from \pref{lem:balancing},
\begin{align*}
    d_i C \sqrt{n_i(t_i)} = \regretbound_i(n_i(t_i)) \leq \regretbound_\star(n_\star(t_i)) + 1 \overset{(iv)}{\leq} 2\regretbound_\star(n_\star(t_i)) \leq 2 d_\star C \sqrt{n_\star(t_i)},
\end{align*}
where $(iv)$ holds because no learner can be eliminated before each learner has been played at least once and thus $\regretbound_\star(n_\star(t_i)) \geq 1$. Rearranging this inequality yields $n_i(t_i) / n_\star(t_i) \leq 2 d_\star^2 / d_i^2$. Analogously, we can show that if $n_i(t_i)$ satisfies $n_i(t_i) = \regretbound_i(n_i(t_i))$, then $n_i(t_i) / n_\star(t_i) \leq 2$. This follows from $n_i(t_i) = \regretbound_i(n_i(t_i)) \leq 2\regretbound_\star(n_\star(t_i)) \leq 2 n_\star(t_i)$. Thus, the sum of play ratios is bounded as
\begin{align*}
    \sum_{i \in \misspidx} \frac{n_i(t_i)}{n_\star(t_i)} \leq \sum_{i \in \misspidx} 2 \vee 2 \frac{d_\star^2}{d_i^2} \leq 2 B d_\star^2~.
\end{align*}
\end{proof}

\paragraph{Linear regret base learners.}
When we instantiate \pref{thm:poly_reg} to the case where candidate regret bounds are linear in $n$ ($\beta = 1$), then the total regret of \pref{alg:balanced_elim_general} is of order $\tilde O(MCd_\star T)$, which is only a factor $M$ worse than the regret bound for the best well-specified learner. The follow result shows that this is still the case when candidate regret bounds come with an additional $\sqrt{n}$ term common to all learners under the additional assumption that no misspecified algorithm has a larger candidate regret bound than the best well-specified learner. This will be useful when \pref{alg:balanced_elim_general} is used with contextual bandits or linear MDP algorithms with misspecified function classes (see \pref{sec:applications}).  

\begin{theorem}\label{thm:linTmissp}
Let the candidate regret bounds for all $M$ base learners be of the form
\begin{align}
    \regretbound_i(n) = C_1\sqrt{n} + \epsilon_i C_2 n \wedge n,
    \label{eqn:linT_regcand}
\end{align}
where $\epsilon_i \in (0, 1]$ and $C_1, C_2 > 1$ are quantities that do not depend on $\epsilon_i$ or $n$.
Then, probability at least $1 - \delta$, the total regret of \pref{alg:balanced_elim_general} is bounded for all rounds $T$ as
\begin{align*}
    \regret(T) =
    O\left(M C_1 \sqrt{T}\sqrt{\ln \frac{M\ln T}{\delta}} + M C_2 \epsilon_\bestallwell T \sqrt{\ln \frac{M\ln T}{\delta}} + B C_1^2\right)~,
\end{align*}
where $\bestallwell \in \wellspidx$ is any well-specified base learner such that $\epsilon_i \leq \epsilon_\bestallwell$ for all misspecified learners $i \in \misspidx$.
\end{theorem}
\begin{proof}
This statement is proven analogously to the generic bound in \pref{thm:poly_reg}, but it makes heavy use of a case-by-case analysis of the different regimes of candidate regret bounds provided in \pref{lem:lin_sqrt_regret_bound} in \pref{app:stoch_proofs}.
\end{proof}

\subsection{Gap-Dependent Regret Bounds}\label{s:gapdependent}
The regret guarantees in the previous section only depend on which learners are well- or misspecified and their presumed regret bounds. In particular, a misspecified learner may violate their presumed regret bound at any time by any amount. However, in many relevant practical cases, a base learner is either well-specified or violates their presumed regret bound by a significant amount. For example in contextual bandits where each base learner has access to a restricted policy class, a learner achieves good $\sqrt{T}$ regret when the optimal policy is contained in its policy class, but has otherwise to suffer linear regret.
We now provide tighter guarantees for \pref{alg:balanced_elim_general} in such cases.
Specifically, we assume that if a learner $j$ is misspecified, its regret is lower-bounded by
\begin{align*}
    \regret_j(t) \geq \Delta_j n_j(t)^{\alpha}
\end{align*}
for all $t$, where $\Delta_j > 0$ and $\alpha$ is strictly larger than both $\frac{1}{2}$ and the presumed regret rate $\beta$ in Eq. (\ref{eqn:poly_reg_candidates}). Since the regret of $j$ grows significantly faster than its presumed regret bound and the regret of the best well-specified learner (that is, $\regret_j$ has a large {\em gap}), we can show that the elimination test in \pref{alg:balanced_elim_general} is triggered after playing learner $i$ for a certain number of times. This allows us to prove the following gap-dependent regret-guarantee:
\begin{restatable}{theorem}{polyreggapthm}\label{thm:poly_reg_gap}
Assume \pref{alg:balanced_elim_general} is used with candidate regret bounds in Equation~\eqref{eqn:poly_reg_candidates} and that the pseudo-regret of all misspcified learners $j \in \misspidx$ is bounded for all $t$ from below as $\regret_j(t) \geq \Delta_j n_j(t)^{\alpha}$, for some constants $\Delta_j > 0$ and $\alpha > \frac{1}{2} \vee \beta$. If $0 < \beta < \frac{1}{2}$ then total regret is bounded with probability at least $1- \delta$ for all $T$ as
\begin{align*}
     \regret(T) = O\left( M d_\star C T^{\beta} + \sum_{i \in \misspidx} 
    C   \left( (2d_\star)^{\frac{1}{\beta} + \frac{1}{\beta(2\alpha - 1)}} + d_\star d_i^{\frac{1}{2\alpha - 1}}
    \right) \left[ \frac{20C }{\Delta_i} \ln \frac{M \ln T}{\delta}\right]^{\frac{1}{2\alpha - 1}}
    \right)~,
\end{align*}
where $\star \in \Wcal$ is any well-specified learner. If instead $\beta \geq \frac{1}{2}$, then the total regret is bounded with probability at least $1- \delta$ for all $T$ as
\begin{align*}
    \regret(T) = O\left( M d_\star C T^{\beta} + \sum_{i \in \misspidx} 
    C \sqrt{ \ln\frac{M \ln T}{\delta}}  \left( d_\star^{\frac{1}{\beta} + \frac{1}{\alpha - \beta}} + d_\star d_i^{\frac{\beta}{\alpha - \beta}}
    \right) \left[ \frac{20C}{\Delta_i} \right]^{\frac{\beta}{\alpha - \beta}}
    \right).
\end{align*}
\end{restatable}
Although the argument of eventually eliminating base learners with a large gap is similar to a gap-dependent analysis is multi-armed bandits, it is important to note that the notion of gap here is a property of the base learner and not (necessarily) of the 
action space at hand.

\pref{tab:summary_gap_bounds} contains a summary of the guarantees in \pref{thm:poly_reg_gap} for the special case where $\alpha = 1$ and $\beta = \frac{1}{3}, \frac{1}{2}$ and $\frac{2}{3}$.
Comparing \pref{thm:poly_reg_gap} to \pref{thm:poly_reg} (or \pref{tab:summary_gap_bounds} to \pref{tab:summary_worst_case}), we see that the multiplicative factor in front of the best well-specified regret bound is only $M$, as compared to the presence of extra $d_\star$ factors without a gap-assumption. Further, while the additive term in \pref{tab:summary_gap_bounds} may have a dependency on a potentially large $d_i$, this term only scales with $T$ as $\ln \ln T$, and is thus virtually constant. Importantly, this yields the optimal scaling in $T$ even when $\beta < \frac{1}{2}$ (see the first line of Table \ref{tab:summary_gap_bounds}) so that the additional $\sqrt{T}$-term occurring in \pref{tab:summary_worst_case} can be avoided. This result is in contrast with existing approaches such as~\cite{pacchiano2020model}, where the $\sqrt{T}$ dependence cannot be avoided.

\begin{table}
    \centering
    \bgroup
\def\arraystretch{1.7}
    \begin{tabular}{@{\hspace{0cm}}l|ll@{\hspace{0cm}}}
    Presumed Bounds $\regretbound_i$ & Gap-Dependent Regret Guarantee of \pref{alg:balanced_elim_general} & Proof \\
    \hline
        $d_i C n^{1/3}$ & 
    $ M \textcolor{DarkGreen}{d_\star C T^{1/3}} + \sum_{i \in \misspidx} \frac{C^2 (d_\star^6 + d_\star d_i)}{\Delta_i} \ln \frac{M \ln T}{\delta}$
    & \pref{thm:poly_reg_gap}
    \\
    $d_i Cn^{2/3}$ &
    $M  \textcolor{DarkGreen}{d_\star C T^{2/3}}
    + \sum_{i \in \misspidx} \frac{C^3(d_\star^{4.5} +d_\star d_i^2)}{\Delta_i^2} \sqrt{\ln \frac{\ln T}{\delta}}
    $
    & \pref{thm:poly_reg_gap}
    \\
    $d_i C\sqrt{n}$ 
    & $M  \textcolor{DarkGreen}{d_\star C \sqrt{T}}
    + \sum_{i \in \misspidx} \frac{C^2(d_\star^4 +d_\star d_i)}{\Delta_i} \sqrt{\ln \frac{\ln T}{\delta}}
    $
    & \pref{thm:poly_reg_gap}
    \\
    \hline
        $d_i C\sqrt{n \ln\frac{n}{\delta}}$ 
    & $M  \textcolor{DarkGreen}{d_\star C \sqrt{T \ln\frac{T}{\delta}}}
    + \sum_{i \in \misspidx} \frac{C^2(d_\star^4 +d_\star d_i)}{\Delta_i} \ln^{3/4}\frac{MT}{\delta} \ln^{3/2} \frac{\ln T}{\delta}
    $
    & \pref{thm:sqrt_reg_gap}
    \end{tabular}
    \egroup
    \caption{Summary of our gap-dependent regret bounds when each misspecified learner has linear pseudo-regret ($\alpha = 1$). Some constant factors are omitted for readability.
    In green is the regret guarantee of the best well-specified learner.} \label{tab:summary_gap_bounds}
\end{table}

%% file: applications.tex
\section{Example Applications}
\label{sec:applications}
\subsection{Brief Review of Contextual Linear Bandits and the \oful\ Algorithm}
One important application of the methods we presented in \pref{sec:balancingandelimination} and  \pref{sec:regretanalysis} is the setting of contextual linear bandits, which we now briefly review. To keep consistency with previous sections, we shall assume here that contexts are drawn i.i.d. from some distribution over context space $\Xcal$. Yet, the algorithmic solutions we present (specifically, the \oful\, algorithm) actually work unchanged even in the more general fixed design or adaptive design scenarios. This will be useful in \pref{sec:adv}, when dealing  with the {\em adversarial} contextual bandit setting.

In the contextual bandit setting, context $x_t$ determines the set of actions $\Acal_t \subseteq \Acal$ that can be played at time $t$. When the bandit setting is {\em linear} the policies we consider are of the form $\pi_{\theta}(x_t) = \arg\max_{a \in \Acal_t}\langle a_t,\theta\rangle$, for some $\theta \in \mathbb{R}^d$, and the class of policies $\Pi$ can then be thought of as a class of $d$-dimensional vectors $\Pi \subseteq \mathbb{R}^d$. Moreover, rewards are generated according to a noisy linear function, that is,
$r_t= \langle a_t, \theta_* \rangle + \xi_t$, where $\theta_* \in \Pi$ is unknown, and $\xi_t$ is a conditionally zero mean $\sigma-$subgaussian random variable. 
We denote the time-$t$ optimal action as $a_t^\star = \argmax_{a \in \Acal_t} \langle a, \theta_\star \rangle$. The learner's objective is to control its pseudo-regret:
\begin{equation*}
    \regret(T) = \sum_{t=1}^T  \langle a_t^\star, \theta_\star\rangle -  \langle a_t, \theta_\star \rangle~.
\end{equation*}

\paragraph{OFUL Algorithm.}
\begin{algorithm}[t]
\textbf{Input: }regularization parameter $\lambda>0 $, confidence scaling $\beta_1, \beta_2, \dots$\\
\For{round $t=1, 2, \dots$}{
    Update regularized least-squares estimator $\hat{\theta}_t$ and covariance matrix $\Sigma_t$\\
    Receive context $x_t$/action space $\Acal_t$\\
    Play optimistic action:
$$    a_t \in  \argmax_{a \in \Acal_t} \max_{\theta \in \mathcal{C}_t } \langle a, \theta \rangle
= \argmax_{a \in \Acal_t}\, \langle \hat \theta_t, a \rangle + \beta_t \|a\|_{\Sigma_t^{-1}}
$$
Receive reward $r_t = \langle a_t, \theta_\star \rangle + \xi_t$~.\\
}
\caption{\oful\, \citep{abbasi2011improved}}
\label{alg:OFUL}
\end{algorithm}
We now recall the relevant components of the OFUL algorithm~\citep{abbasi2011improved} shown in \pref{alg:OFUL}. Instances of this algorithm will play the role of base learners in subsequent sections.
The OFUL algorithm proceeds by computing a regularized least-squares (RLS) estimator $\hat{\theta}_t$ of the true parameter $\theta_\star$ using the data collected so far:
\begin{align}
    \hat{\theta}_t := \Sigma_{t}^{-1} \left(  \sum_{\ell=1}^{t-1} a_{\ell}\, r_{\ell}\right) \quad \textrm{where} \quad \Sigma_{t} = \lambda \mathbb{I} + \sum_{\ell=1}^{t-1} a_\ell a_\ell^\top~.
    \label{eqn:oful_ls}
\end{align}
Here, $\Sigma_t$ is the regularized covariance matrix of the played actions up to the beginning of round $t$ with regularization parameter $\lambda$, and $\mathbb{I}$ denotes the $d\times d$ identity matrix. 
Using $\hat{\theta}_t$ and $\Sigma_t$, \oful\ 
proceeds by computing a confidence ellipsoid
\begin{equation}\label{eq:ellipsoid_OFUL}
    \mathcal{C}_t := \{\theta \,:\, \| \theta - \hat{\theta}_t \|_{\Sigma_{t}} \leq \beta_t \}
\end{equation}
that should contain the optimal parameter $\theta_\star$. We will discuss a choice of the (possibly data-dependent) scaling factor $\beta_t \in \RR_+$ below that ensures that this happens in all rounds with high probability.
\pref{alg:OFUL} now plays any action that achieves highest expected return in what we refer to as the optimistic model
\begin{equation}\label{eqn:oful_beta}
    \tilde{\theta}_t = \argmax_{\theta \in \mathcal{C}_t} \max_{a \in \Acal_t } \langle a, \theta\rangle~. 
\end{equation}
This choice of action is equivalent to picking $a_t \in \argmax_{a \in \Acal_t}\, \langle \hat \theta_t, a \rangle + \beta_t \|a\|_{\Sigma_t^{-1}}$.

 We define the event that the above-mentioned ellipsoidal confidence set $\Ccal_t$ contains $\theta^*$ at all times $t \in \NN$ as
\begin{equation}\label{eq:gcal_OFUL}
    \Ecal = \{ \theta_* \in \mathcal{C}_t, \quad  \forall t \in \mathbb{N} \}~.
\end{equation}
In this event $\Ecal$, the optimistic model $\tilde \theta_t$ indeed gives rise to an optimistic estimate of the expected reward in each round
\begin{equation}\label{equation::optimism_OFUL}
    \langle a_t, \tilde{\theta}_t\rangle \geq \max_{a \in \Acal_t} \langle a, \theta_\star \rangle = \langle a_t^\star, \theta_\star\rangle~. 
\end{equation}
\citet{abbasi2011improved} show that the following choice for $\beta_t$ is sufficient to make $\Ecal$ happen with high probability:
\begin{lemma}[Theorem 1 in \cite{abbasi2011improved}] 
\label{lem:yasin}
For any $\delta \in (0, 1)$, let the confidence scaling be
\begin{align}\label{eq:beta_OFUL}
    \beta_t &:= 
    \sqrt{2\sigma^2  \ln\left( \frac{\det(\Sigma_t)^{1/2} \det(\lambda I)^{-1/2}}{\delta}\right)} + \sqrt{\lambda}S 
    \leq
    \sqrt{\sigma^2 d \ln\left( \frac{1 + tL^2/\lambda}{\delta}\right)} + \sqrt{\lambda}S~ 
\end{align}
where $S$ is a known bound on the parameter norm $\max_{\theta \in \Pi} \| \theta \|_2$ and
$L$ is a known bound on the action norm in all rounds, i.e., $\max_{a \in \Acal_t} \| a \|_2 \leq L$ for all $t$.
Then $\theta_\star$ is contained in the confidence ellipsoid with high probability, i.e., $\mathbb{P}\left( \Ecal \right) \geq 1-\delta$.
\end{lemma}
In event $\Ecal$, one can show that the regret of \pref{alg:OFUL} is bounded for all $t \in [T]$ as
\begin{align*}
    \regret(t) %
    \leq %
    2 \beta_{\max}
    \sqrt{ dt \left(1 + \frac{L^2}{\lambda
 }\right) \ln \frac{d \lambda + tL}{d \lambda}}~,
\end{align*}
where $\beta_{\max} = \max_{k \in [t]} \beta_k$.
We reproduce a slightly more general version of the standard proof for this regret bound in \pref{lem:LinUCB_regret_guarantee} in the appendix.
The right side of the above inequality will play the role of our presumed regret bound $\regretbound(n_i(t))$ when \oful\, is used as a base learner. 

In the rest of this section, we present a number of applications of our balancing and elimination machinery to the case where the base learners are instances of the \oful\, algorithm.

\subsection{Linear Bandits with Nested Model Classes}\label{sec:linbannested}

We can apply our regret bound balancing algorithm to linear bandits where the true dimensionality $d_\star$ of the model $\theta_\star$ is unknown a-priori. In this standard scenario, considered by many recent papers in the model selection literature for bandit
algorithms~\citep[e.g.][]{foster2019model, pacchiano2020model}, the learner chooses among actions $\Acal_t \subseteq \RR^{d_{\max}}$ of dimension
$d_{\max}$ but only the first $d_\star$ dimensions are relevant (that is, $(\theta_\star)_i = 0$ for $i > d_\star$).

One can learn in this setting as follows: 
We use $\log_2 d_{\max}$ instances of OFUL as base learners\footnote{We here assume that $d_\star$ and $d_{\max}$ are powers of 2 for convenience but our results also hold generally up to a constant factor of $2$.}. Each instance $i$ first truncates the actions to dimension $d_i = 2^i$ and then only computes the least-squares estimate and confidence ellipsoid in $\RR^{d_i}$.
Based on the OFUL regret guarantees in the previous section, we use $\regretbound_i(n) = d_i C \sqrt{n} \wedge n$ as putative regret bounds, with constant $C$ set to  
\begin{align*}
    C = 2\left( \sigma + \sqrt{\lambda}S \right)
    \sqrt{  \left(1 + \frac{L^2}{\lambda
 }\right) \ln \left(\frac{1 + TL^2 / \lambda}{\delta}\right) \ln \frac{\lambda + TL}{\lambda}}~.
\end{align*}
For convenience, we here assume the time horizon $T$ is known and $\ln T$ terms can therefore be absorbed into the constant $C$ common to all base learners, but any-time versions are also possible  by setting $n = T$ above at which the regret bound scales as $\sqrt{n}\ln n$ 
(see \pref{thm:sqrt_reg} in appendix).
By the regret guarantee of OFUL discussed in the previous section, with probability $1 - M\delta$, any base learner $i$ such that $d_i < d_{\star}$ will be misspecified, while all remaining $i$ are well specified.

More specifically, we have 
$M = O(\ln d_{\max})$-many base learners, out of which $B =  O(\ln d_{\star})$ are misspecified. Then a direct application of \pref{thm:poly_reg} with $\beta = 1/2$ gives
\begin{align*}    
\regret(T) &= O\left(\left(\ln d_{\max} + d_\star\sqrt{\ln d_\star}\right) d_\star C \sqrt{T}
\right)
\approx O\left(\left(\ln d_{\max} + d_\star\sqrt{\ln d_\star}\right) d_\star \sqrt{T} \ln T
\right),
\end{align*}
where the second expression only retains dependencies on $T$, $d_\star$ and $d_{\max}$.

If further all misspecified learners suffer linear regret $\regret_i(t) \geq \Delta n_i(t)$ for some $\Delta > 0$ (e.g. since they cannot represent the observed rewards, they may converge to playing a strictly suboptimal action for most contexts), then applying \pref{thm:poly_reg_gap} yields
\begin{align*}
    \regret(T) &= O \left( \ln(d_{\max}) d_\star C \sqrt{T} + \ln (d_\star)\frac{ C^2 d_\star^4}{\Delta} \sqrt{\ln \frac{\ln T}{\delta}}\right) \\
    &\approx O \left( \ln(d_{\max}) d_\star \sqrt{T} \ln T + \frac{ d_\star^4 \ln d_\star }{\Delta}(\ln T)^2 \sqrt{\ln \ln T}\right)~,
\end{align*}
where the second expression again only shows dependencies on $T$, $d_\star$, $d_{\max}$ and $\Delta$. Notice that, as $T$ grows large, the main term of the above bound becomes $d_\star \sqrt{T}$, up to log factors. This is precisely the bound we would achieve had we known in advance dimension $d_{\star}$, and just played the associated base \oful\, from beginning to end. 

\begin{remark}
A standard goal in model selection is to obtain sub-linear regret bounds even in the case where the model complexity of the target class is allowed to grow sub-linearly with $T$ -- see, e.g., the discussion in \citep{foster2019model}.
In our case, this would be obtained by regret bounds of the form $d_\star^{\alpha}\, T^{1-\alpha}$, for some $\alpha \in (0,1)$, for example a bound of the form
$\sqrt{d_\star\, T}$.
It is worth observing that in the setting considered in this paper this is an impossible goal to achieve since, unlike \cite{foster2019model}, we are dealing with {\em infinite} action spaces, and the best one can hope for in this case is indeed $d_\star\sqrt{T}$ (see Section 2 in \cite{rt10}).
\end{remark}

\subsection{Linear Markov Decision Processes with Nested Model Classes}\label{sec:linearmdpnested}
We can instantiate the regret bound in \pref{thm:poly_reg} ($\beta = 1/2$) to the episodic linear MDP setting of \citet{jin2020provably}, again with nested feature classes of doubling dimension, as in \pref{sec:linbannested}. 
Here, each round $t$ of \pref{alg:balanced_elim_general} corresponds to one episode of $H$ time steps in the MDP, and contexts $x_t$ are the initial state of the episode in the MDP.
\citet{jin2020provably} prove that their \textsc{LSVI-UCB} algorithm achieves regret $O(H^2\sqrt{ d^3 K}\ln(dK / \delta))$ after $K$ episodes when used with a realizable function class of dimension $d$. We deploy $M = O(\ln d_{\max})$ instances of \textsc{LSVI-UCB} as base learners with presumed regret bounds
\begin{align*}
    \regretbound_i(n) = Hn \wedge H^2 \sqrt{d_i^3 n}\ln(d_{\max}T / \delta).
\end{align*}
Since the total reward per episode (= round) is in $[0, H]$ instead of $[0, 1]$ in this setting, we scale the regret bound as well as the constant $c$ in \pref{alg:balanced_elim_general} by $H$. By \pref{thm:poly_reg} the total regret of \pref{alg:balanced_elim_general} after $T$ episodes is bounded as
\begin{align*}
    \regret(T) &= O \left(
    \left(\sqrt{d_\star^3 \ln d_{\star}} + \ln d_{\max}\right)H^2 \sqrt{d_\star^3 T}\ln(d_{\max}T/\delta)
    \right)
\end{align*}
with probability $1 - M \delta$. Similar to the contextual bandit setting above, we can achieve a tighter bound if all misspecified learners suffer linear regret $\regret_i(t) \geq \Delta n_i(t)$ for some $\Delta > 0$. Then applying \pref{thm:poly_reg_gap} yields
\begin{align*}
    \regret(T) &= 
    O \left(
    H^2 \sqrt{d_\star^3 T}\ln (d_{\max})\ln(d_{\max}T/\delta)
    + \frac{H^4 d_\star^6}{\Delta} \ln(d_{\max} T/ \delta)^2 \sqrt{\ln \frac{\ln T}{\delta}}
    \right)\, 
\end{align*}
which, up to log factors and lower order terms, again coincides with the regret bound of the best base learner in hindsight.

\subsection{Linear Bandits and MDPs with Unknown Approximation Error}\label{sec:unknownapx}
\citet{zanette2020learning} presents an algorithm for learning a good policy in episodic MDPs where the value functions are all close to a linear feature space of dimension $d$. Their algorithm admits a high-probability regret bound of order\footnote{The $\tilde O$ notation is similar to the $O$-notation but hides poly-logarithmic dependencies.} $\tilde O(Hd\sqrt{T} + H\sqrt{d} \epsilon T)$ for all $T$ when a bound $\epsilon$ on the inherent Bellman error is known a-priori. For details of the setting and the exact definition of inherent Bellman error see  \citet{zanette2020learning}. Unfortunately, in most practical applications, one does not know $\epsilon$ ahead of time and picking a conservative value (large $\epsilon$) makes the algorithm over-explore and suffer large regret.

We can address this limitation by applying \pref{alg:balanced_elim_general} with several instances of their algorithm as base-learners, each associated with a certain value of the inherent Bellman error $\epsilon_i = \frac{2^{1-i}}{\sqrt{d}}$ and the putative regret bound $\regretbound_i(n) =  (CHd\sqrt{n} + CH \sqrt{d}\epsilon_i n) \wedge Hn$ for an appropriate value $C$ that depends at most logarithmically on $d, T$ or $H$. It is sufficient to use $M = \lceil 1 + \frac{1}{2} \log_2 (T / d^2) \rceil$ base learners since the putative regret bound of learner $1$ (with $\epsilon_1 = 1/\sqrt{d}$ and $\regretbound_1(n) \geq Hn$) always holds, while the putative regret bound of learner $M$ is at most $\regretbound_M(T) \leq 2 CH d \sqrt{T}$, which is a constant factor worse than the regret when $\epsilon = 0$.

By \pref{thm:linTmissp}, the total regret of \pref{alg:balanced_elim_general} with these base learners is 
\begin{align*}
    \regret(T) &=O \left( MCH(d \sqrt{T} + \sqrt{d}\epsilon_\star T) \sqrt{\ln \frac{M \ln T}{\delta}} + B C^2H^2 d^2\right)\\
    &= \tilde O \left( H d \sqrt{T} + H \sqrt{d}\epsilon_\star T  + H^2 d^2\right)
\end{align*}
with probability $1 - M \delta$.
Hence, up to at most logarithmic factors and a lower-order additive term, our model-selection framework can recover the best regret bound without requiring knowing the inherent Bellman error ahead of time.
Notice also that the special case $H = 1$ recovers the standard linear bandit setting and the algorithm by \citet{zanette2020learning} reduces to OFUL with a confidence ellipsoid that accounts for $\epsilon_i$. In this bandit case $\epsilon_\star$ is the absolute approximation error of expected rewards. 

Recently, \citet{foster2020adapting} have shown that an adaptation to unknown approximation errors $\epsilon_\star$ is possible in contextual bandits, but their model-selection approach requires base learners that work with importance weights, and whose importance-weighted regret admits a favorable dependency on $\epsilon_i$. Here we have shown that a similar result (up to logarithmic factors) can be achieved with standard optimistic base learners such as \oful.
Our result also matches the regret-guarantee by \citet{pacchiano2020model} but does not require their smoothing procedure for base-learners.
Importantly, our result proves that an adaptation to unknown approximation errors $\epsilon_\star$ is also possible without any modification to base learners in the MDP setting where base-learners that achieve the importance-weighted regret guarantee required by \citet{foster2020adapting} are (still) unavailable. Note also that our framework is not specific to instances of the algorithm by \citet{zanette2020learning} as base learners. Our model selection algorithm can, for example, also be used with approximate versions of LSVI-UCB by \citet{jin2020provably} and achieve similar regret guarantees in their setting and for their notion of approximation error.

\subsection{Confidence parameter tuning in \oful}\label{sec:optoful}
A standard problem that arises in the practical deployment of contextual bandit algorithms like \oful\, is that they are extremely sensitive to the tuning of their upper-confidence parameter ruling the actual trade-off between exploration and exploitation.
The choice of confidence parameter from \pref{lem:yasin} ensures high-probability regret guarantee but is often too conservative. This can for example be the case when the actual noise variance is smaller than the assumed $\sigma^2$ variance. While there are concentration results (empirical Bernstein bounds) that can adapt to such fortunate low-variance noise for scalar parameters (e.g., in unstructured multi-armed bandits), such adaptive bounds are still unavailable for least-squares estimators. Empirically, choosing smaller values for $\beta_1, \dots, \beta_T$ can often achieve significantly better performance but comes at the cost of losing any theoretical performance guarantee. Our model-selection framework can be used to tune the confidence parameter online and simultaneously achieve a regret guarantee.

We will now look at ways to
compete against the instance of the \oful\, algorithm which is equipped with the optimal scaling of its upper-confidence value, in the sense of the following definition:
\begin{definition}
Denote by $\bar \beta_{t}$ the confidence-parameter choice from \pref{lem:yasin} and let $\kappa \in \RR_+$ be a scaling factor.
Further, let ${\hat \theta_{S}}(\kappa)$ and $\Sigma_S(\kappa)$ be the iterates of least squares estimator and covariance matrix obtained by running \oful\  with scaled confidence parameters $(\kappa \bar \beta_t)_{t \in \NN}$ on a subset of rounds $S \subseteq [T]$. Then, for a given range $[\kappa_{\min}, 1]$, the optimal confidence parameter scaling for \oful\, is defined as
\begin{align*}
    \kappa_\star = \min_{\kappa \in [\kappa_{\min}, 1]} \max_{S \subseteq [T]} \frac{\|{\hat \theta_{S}}(\kappa) - \theta_\star\|_{\Sigma_{S}(\kappa)^{-1}}}{\bar \beta_{|S|}}~. 
\end{align*}
\end{definition}
In words, the optimal $\kappa_\star$ is the smallest scaling factor of confidence parameters that ensures that no matter to what subset of rounds we would apply \oful\ to, the optimal parameter $\theta_{\star}$ is always contained in the confidence ellipsoid. 
Observe that $\kappa_{\star}$ is a random quantity, i.e., $\kappa_\star$ is the best scaling factor for the given realizations in hindsight. \pref{lem:yasin} ensures that $\PP(\kappa_\star \leq 1) \geq 1 - \delta$ and empirical observations suggest that $\kappa_\star$ is much smaller in many events and bandit instances.

Now, \pref{lem:LinUCB_regret_guarantee} in \pref{app:ancillary} ensures that \oful\ with confidence parameters $\kappa \bar \beta_t$ admits  a regret bound of the form\footnote
{
For simplicity of presentation, we set here $\lambda = 1$ and disregarded the dependence on other parameters like $L$, $S$, and $\sigma$.
}
$\regret(n) \lesssim \kappa d \sqrt{ n }\ln(n) \wedge n$ if $\kappa \geq \kappa_\star$. Since $\kappa_\star$ is unknown, we run \pref{alg:balanced_elim_general} with $M$ instances of \oful\ as base learners, each with a scaling factor $\kappa_i = 2^{1-i}$, $i = 1,\ldots, M$, and putative regret bound $\regretbound_i(n) \approx \kappa_i d \ln(T) \sqrt{n} \wedge n$.
Note that it is sufficient to use $M = 1 + \log_2 \frac{1}{\kappa_{\min}}$~.

Then, by \pref{thm:poly_reg} (with $\beta = 1/2$ therein), the regret of \pref{alg:balanced_elim_general} is bounded with probability at least $1 - \delta$ as
\begin{align*}
    \regret(T) &\lesssim \left(M + \sqrt{B}  \frac{\kappa_i}{\kappa_{\min}} \right) \regretbound_\star(T)\\
    &= O\left( \left(\frac{\kappa_\star}{\kappa_{\min}} \sqrt{\ln \frac{\kappa_\star}{\kappa_{\min}}} + \ln \frac{1}{\kappa_{\min}}  \right) %
    \kappa_\star d \ln(T) \sqrt{ T}
    \right).
\end{align*}
Note that this is a random and problem-dependent bound because so is $\kappa_\star$. In cases where $\kappa_{\star} \lesssim \sqrt{\frac{\kappa_{\min}}{\ln(1 / \kappa_{\min})}}$, this bound strictly improves on the standard \oful\, bound relying on confidence scaling $\kappa = 1$, which is often way too conservative in practice.

%% file: adversial_ctx.tex
\section{Extension to Adversarial Contexts}
\label{sec:adv}
In this section, we show that the regret balancing and elimination principle can also be used for model selection when the contexts $x_t$ are generated in an adversarial manner. This requires slightly stronger assumptions on the base learners, which hold in many settings when we select between a hierarchy of optimistic learners such as OFUL or LSVI-UCB. For the sake of concreteness, we present our extension of the regret balancing and elimination algorithm to adversarial contexts for the setting from \pref{sec:linbannested}, but our technique for adversarial contexts can be easily adapted to all other bandit applications discussed in \pref{sec:applications} and likely to episodic MDP settings with adversarial start states as well.

Let us briefly recall the setting from \pref{sec:linbannested}. We consider the problem of linear bandits and are given $M$ instances of OFUL as base learners. Each instance $i$ considers only on the first $d_i = 2^i$ dimensions of the actions, with $d_1 < d_2 < \dots < d_M$. Since the entries of the true parameter $\theta_\star$ are $0$ for all dimensions above $d_{i_\star}$, where $i_\star \in [M]$ is an unknown index, all learners $i_\star, i_\star + 1, \dots M$ are well-specified with high probability. We focus our analysis on the event $\Ecal$ where this is the case. Unlike \pref{sec:linbannested} where contexts are assumed to be drawn i.i.d., we here consider the setting where contexts $x_t$ (corresponding to the action set $\Acal_t$ at round $t$) are generated adversarially. Since each base learner operates only in a lower-dimensional subspace, we allow the bounds on the action norm $L_i$, the bound on the parameter norm $S_i$ and the range of expected return $R^{\max}_i$ to vary per base learner $i$ (potentially depend on the number of dimension $d_i$) but for the sake of simplicity, we assume that all learners use regularization parameter $\lambda = 1$.

\pref{alg:balanced_elim_general}, which assumes stochastic contexts, compares upper- and lower confidence bounds on the optimal return value $\mu^\star$ obtained from learners that were executed on two disjoint subsets of rounds to determine misspecification. This strategy does not work with adversarial contexts since the optimal policy that an algorithm could have achieved depends on the contexts in the rounds that it was played. One algorithm may only have seen "bad" contexts with low $\mu^\star_t$, while another may only encountered favorable contexts with high $\mu^\star_t$. A direct comparison is therefore meaningless.

To be able to handle adversarial contexts and address this challenge, we modify our regret balancing and elimination algorithm in two ways: (1) we randomize the learner choice for regret balancing and (2) we change the misspecfication test to compare upper and lower confidence bounds on the optimal policy value of \emph{all} rounds played to far. The resulting algorithm is presented in \pref{alg:adversarial_model_selection} which operates in epochs where the subroutine in \pref{alg:adversarial_epoch_balancing}
is executed. 
We start by discussing the regret balancing subroutine in the next section before presenting the main algorithm and its regret guarantee afterwards.

\subsection{The Epoch Balancing Subroutine}
\label{section:epoch_balancing_algorithm}

\begin{algorithm}[t]
   \textbf{Input:} set of learners $\Ical$%
   
\For{round $t=1, 2, \dots $}{
    Receive context $x_t$
    
    \ForEach{learner $i \in \Ical$}{
    Ask learner $i$ for a lower bound $B_{t, i}$ on the value of its proposed action
    }
    Sample $i_t \sim p \propto \frac{1}{z_i}  \textrm{ for } i \in \Ical$ \hspace{2cm} (see Equation~\eqref{eqn:prob_def}) \\
    Play learner $i_t$ and receive reward $r_t$
    
    Update base learner $i_t$ with $r_t$
    
    Test for misspecification by checking
    $\displaystyle \sum_{i \in \Ical} [U_i(t) + \regretbound_i(n_i(t))] + c \sqrt{t \ln\frac{\ln(t)} {\delta}} < \max_{i \in \Ical} \sum_{k = 1}^t B_{k, i}
    $
    \\
    \If{\textrm{above condition is triggered}}{
    \textbf{Return} \tcp*{At least one learner must be misspecified}
    }
    
}
\caption{EpochBalancing}
\label{alg:adversarial_epoch_balancing}
\end{algorithm}

This subroutine in \pref{alg:adversarial_epoch_balancing} takes in input a set of active base learners $\Ical = \{s, s+1, \ldots, M\}$ and ensures by randomized regret bound balancing that its total regret is controlled for all rounds until it terminates. 

In addition to the putative bound $R_i$ on its regret, \pref{alg:adversarial_epoch_balancing} requires that each learner $i$ can also provide a lower-confidence bound on $\EE[r_t | a_{t, i}, x_t]$, the expected reward of the action it would play in the current context $x_t$. Since each base learner $i$ is an instance of OFUL, we can choose these bounds at round $t$ as
\begin{align*}
R_i(n_i(t)) &= 2\sum_{k \in T_i(t)} \left(\beta_{ k, i} \|a_{k, i}\|_{\Sigma_{k, i}^{-1}} \wedge R^{\max}_i \right) & \textrm{and}  \\
    B_{t, i} &= \left( \langle  \wh \theta_{t, i}, a_{t, i} \rangle - \beta_{t, i} \|a_{t, i}\|_{\Sigma_{t, i}^{-1}} \right) \vee - R^{\max}_i
\end{align*}
where $R^{\max}_i \in [1,   L_i S_i]$ is the range of expected returns\footnote{We specifically assume that $\EE[r_t | a_t, x_t] \in \left[- R^{\max}_{\star}, +  R^{\max}_\star \right]$ where $\star$ is the smallest base learner whose model class contains the optimal parameter $\theta_\star$.} and $L_i \geq \max_{t} \|a_{t, i} \|$ and $S_i \geq \|\theta^\star\|$ are the norm bounds used by the OFUL base learners. Further,   
$\wh{\theta}_{t,i}$, $\Sigma_{t, i}$ and $\beta_{t,i}$ are the parameter estimate (Eq.~\ref{eqn:oful_ls}), the covariance matrix (Eq.~\ref{eqn:oful_ls}) and the ellipsoid radius (Eq.~\ref{eqn:oful_beta}) of base learner $i$ at time $t$, respectively. 
In similar spirit, 
\begin{align*}
    a_{t, i} \in \argmax_{a \in \Acal_{t}} \langle  \wh \theta_{t, i}, a \rangle + \beta_{ t, i} \|a_{t, i}\|_{\Sigma_{t, i}^{-1}} 
\end{align*}
denotes the action that base learner $i$ would take at time $t$. Note that we mean here the truncated actions and covariance matrix in $\RR^{d_i}$ and $\RR^{d_i \times d_i}$.

At each round $t$, \pref{alg:adversarial_epoch_balancing} first requests these bounds from each base learner to be later used in the misspecification test. The algorithm then selects one of the base learners in $\Ical$ by sampling an index $i_t \sim \operatorname{Categorical}(p)$ from a categorical distribution with probabilities
\begin{align}
p_i = \frac{1/ z_i}{\sum_{j \in \Ical} 1/z_j}~,~\qquad\textrm{where } z_i = (d_i^2 + d_i S_i^2)\left(R^{\max}_i \wedge L_i^2 \right)\qquad \textrm{for } i \in \Ical~.\label{eqn:prob_def}
\end{align}
Since the regret of OFUL scales roughly at a rate of $\sqrt{z_i T}$, this learner selection rule approximately equalizes the regret of all learners in expectation. 
The algorithm proceeds by playing the action proposed by $i_t$, gathering the associated reward $r_t$, and updating $i_t$'s internal state.\footnote{We may also pass on the observation all base learners when base learners can accept \emph{off-policy} samples (which do not necessarily come from the proposed action), as is the case for OFUL.} 
Finally, \pref{alg:adversarial_epoch_balancing} performs a misspecification test and terminates if this test triggers.  We refer to the execution of \pref{alg:adversarial_epoch_balancing} as an epoch.

Unlike the misspecification test in \pref{alg:balanced_elim_general} which considers the hypothesis that a \emph{specific} learner $i$ is well specified, the misspecification test in \pref{alg:adversarial_epoch_balancing} tests the hypothesis that \emph{all} active learners are well-specified. 
If all OFUL learners $i \in \Ical$ are well-specified, in the sense that their ellipsoid confidence sets contain $\theta_\star$ for all rounds $t$ so far, then each $B_{t,i}$ is also a lower-bound on the optimal value in round $t$, since
\begin{align*}
    B_{t,i} \leq \EE[r_t | a_{t, i}, x_t] \leq \max_{a \in \Acal_t }  \EE[r_t | a, x_t] = \mu^\star_t\ .
\end{align*}
Hence, the right-hand side of the misspecification test in \pref{alg:adversarial_epoch_balancing} is a lower-bound on the optimal value of all rounds to far, that is, it satisfies $\max_{j \in \Ical} \sum_{k = 1}^t B_{k, j} \leq  \sum_{k=1}^t \mu_t^\star$.
Similarly, when all learners are well-specified and satisfy their putative regret bounds, then the left-hand side of the misspecification test is an upper-bound on $\sum_{k=1}^t \mu_k^\star$. We can see this as follows. First, by basic concentration arguments, the realized rewards cannot be much smaller than their conditional expectations with high probability, that is, $\sum_{i \in \Ical} U_i(t) \geq \sum_{k=1}^t \EE[r_t | a_t, x_t] - c \sqrt{t \ln\frac{\ln(t)} {\delta}}$. This implies that
\begin{align*}
    &\sum_{i \in \Ical} [U_i(t) + \regretbound_i(n_i(t))] + c \sqrt{t \ln\frac{\ln(t)} {\delta}} \\
    &\geq \sum_{k=1}^t \EE[r_t | a_t, x_t] + \sum_{i \in \Ical}  \regretbound_i(n_i(t))
    = \sum_{i \in \Ical} \left[ \sum_{k \in T_i(t)} \EE[r_t | a_t, x_t] + \regretbound_i(n_i(t))\right]
    \\&\geq 
    \sum_{i \in \Ical} \left[ \sum_{k \in T_i(t)} \EE[r_t | a_t, x_t] + \regret_i(t)\right]
    = \sum_{i \in \Ical} \sum_{k \in T_i(t)} \mu_k^\star
    = \sum_{k=1}^k \mu_k^\star,
\end{align*}
where the last inequality holds because $R_i(n_i(t)) \geq \regret_i(t)$ when $i$ is well-specified.
Thus, if all learners are well-specified, the misspecification test cannot trigger (with high probability). The following theorem formalizes this argument:
\begin{restatable}{theorem}{advepochnonterminate}
\label{thm:adv_epoch_balancing_termination}
With probability at least $1 - \delta$, \pref{alg:adversarial_epoch_balancing} does not terminate if all base learners are well-specified and their elliptical confidence sets contain $\theta^\star$ at all times.
\end{restatable}

Therefore, if the test does trigger, at least one learner in $\Ical$ has to be misspecified, that is, either their putative regret bound $R_i$ or a lower bound $B_{k,i}$ does not hold.
However, until the test triggers, the condition in the test is sufficient to control the regret as the following theorem formalizes.

In this result, we assume that the base learner regret bounds $z_i$ (see Eq.~\eqref{eqn:prob_def}) are sufficiently apart, i.e., $2 z_i \leq z_{i+1}$ holds for all $i \in \Ical \setminus \{M\}$. Note that this assumption can always be ensured by first filtering the base learners. This filtering can increase the regret by at most a factor of $2$. 
\begin{restatable}{theorem}{advepochregret}
\label{thm:adv_epoch_balancing_regret}
Assume that \pref{alg:adversarial_epoch_balancing} is run with instances of OFUL as base learners that use different dimensions $d_i$ and norm bounds $L_i, S_i$ with $2 z_i \leq z_{i+1}$ (see Eq.~\eqref{eqn:prob_def}). All base learners use expected reward range $R^{\max}_i = 1$ and $\lambda = 1$. Denote by $\star$ the smallest index of the base learner so that all base learners $j \in \Ical$ with $d_j \geq d_\star$ are well-specified and their elliptical confidence sets always contain the true parameter. Then, with probability at least $1 - 2 \delta$, the regret is bounded for all rounds $t$ until termination as
\begin{align*}
    \regret(t) \leq \tilde O 
    \left( 
    \left(d_\star + \sqrt{d_\star}S_\star  + |\Ical| \right)
    \textcolor{DarkGreen}{(d_\star + \sqrt{d_\star}S_\star)\sqrt{t}}
    \right)
\end{align*}
\end{restatable}
Here, we highlighted the regret bound of the single best well-specified learner $\star$ in green. 
We here assumed that the range of expected rewards is known and $1$. If this is not the case and we have to rely on the expected reward range induced by the vector norms $L_i$ and $S_i$, then we have an additional lower-order term: 
\begin{restatable}{theorem}{advepochregretls}
\label{thm:adv_epoch_balancing_regret_ls}
Assume that \pref{alg:adversarial_epoch_balancing} is run with instances of OFUL as base learners that use different dimensions $d_i$ and norm bounds $L_i, S_i$ and $R^{\max}_i = L_i S_i$ with $2 z_i \leq z_{i+1}$ (see Eq.~\eqref{eqn:prob_def}). Denote by $\star$ the smallest index of the base learner so that all base learners $j \in \Ical$ with $d_j \geq d_\star$ are well-specified and their elliptical confidence sets always contain the true parameter. Then, with probability at least $1 - 2 \delta$, the regret is bounded for all rounds $t$ until termination as
\begin{align*}
    \regret(t) \leq \tilde O 
    \left( 
    \left(d_\star L_\star  + \sqrt{d_\star}S_\star L_\star  + |\Ical| \right)
    \textcolor{DarkGreen}{(d_\star + \sqrt{d_\star}S_\star)L_\star \sqrt{t}} + \sum_{i \in \Ical} L_i S_i
    \right)\ .
\end{align*}
\end{restatable}

The proofs of \pref{thm:adv_epoch_balancing_regret_ls} and \pref{thm:adv_epoch_balancing_regret} are similar to the proof of \pref{thm:general_regbound} but requires a randomized version of the standard elliptical potential lemma that we prove in \pref{lem:elliptical_random}.

\subsection{Main Algorithm}
\label{sec:adv_mainalg}

\begin{algorithm}[t!]
\For{$s=1, \ldots , M$}{

 $%
 \mathrm{EpochBalancing}\left(\{s, s+1. \dots, M\} \right) $ in \pref{alg:adversarial_epoch_balancing}\\
}
\caption{Regret Bound Balancing and Elimination with Adversarial Contexts}
\label{alg:adversarial_model_selection}
\end{algorithm}

We now show how to obtain a robust model selection algorithm for adversarial contexts with the help of  the Epoch Balancing subroutine from the previous section. 
Since \pref{thm:adv_epoch_balancing_regret} guarantees that the regret of Epoch Balancing is controlled in each epoch, all that is left it to ensure that the number of epochs is small. When \pref{alg:adversarial_epoch_balancing} terminates, we know that one of the base learners must have been misspecified but we do not know which one. We here use the hierarchy of base learners: It is safe to remove the learner $i_{\min} = \min_{i \in \Ical} d_i$ with the smallest dimension as its model class is a subset of the model classes of other base learners. Thus, if there is a model class that fails to contain $\theta^\star$, this must also be the case for $i_{\min}$. 
Therefore, our main algorithm shown in \pref{alg:adversarial_model_selection} calls Epoch Balancing (\pref{alg:adversarial_epoch_balancing}) repeatedly and removes the smallest index from the active learner set each time. %

Note that once $d_i \geq d_\star$ for all $i \in \Ical = \{s, s+1, \dots, M\}$, Epoch balancing will not terminate with high probability because all remaining learners are well-specified and their bounds hold (see \pref{thm:adv_epoch_balancing_termination}). Therefore, there can only be $i_\star \leq M$ epochs where $d_{i_\star} = d_\star$ and the total regret $\regret(T)$ of \pref{alg:adversarial_model_selection} is just the sum of the regret in each epoch up to the total number of $T$ rounds.
We denote by $t^{(s)}(T)$ the total number of rounds in the first $s$ epochs after a total of $T$ rounds. Note that $t^{(s)}(T)$ are stopping times. The regret in the $s$-th epoch is referred to as $\regret^{(s)}( t^{(s)}(T) - t^{(s-1)}(T))$ where $t^{(s)}(T) - t^{(s-1)}(T)$ is the number of rounds in episode $s$.
Therefore, we can write the total regret as
\begin{align}\label{equation::epoch_decomposition_regret}
    \regret(T) = \sum_{s=1}^{M} \regret^{(s)}(t^{(s)}(T) - t^{(s-1)}(T)) ~. 
\end{align}
The regret incurred within each epoch can be bound using \pref{thm:adv_epoch_balancing_regret}, which yields the main result of this section:

\begin{theorem}[Model Selection for Adversarial Contexts in Stochastic Linear Bandits]\label{theorem::main_adversarial}

Assume that \pref{alg:adversarial_model_selection} is run with instances of OFUL as base learners that use different dimensions $d_i$ and norm bounds $L_i, S_i$ with $2 z_i \leq z_{i+1}$ (see Eq.~\eqref{eqn:prob_def}). All base learners use regularizer $\lambda = 1$. 
With probability at least $1-3(M+1)\delta$ the total regret of \pref{alg:adversarial_model_selection} is bounded for all rounds $T \in \NN$ as
\begin{align*}
     \regret(T) = \tilde{O}\left( \left(\sqrt{B}d_\star + \sqrt{B d_\star}S_\star  + \sqrt{B} M \right)\textcolor{DarkGreen}{(d_\star + \sqrt{d_\star}S_\star)\sqrt{  T} }\right)\ ,
\end{align*}
if base learners use a common expected reward range $R^{\max}_i = 1$. Here, $B$ are the number of base learners that use a misspecified model that cannot represent $\theta_\star$, If base learners use instead $R^{\max}_i = L_i S_i$, then the regret bound is
\begin{align*}
     \regret(T) = \tilde{O}\left( \left(\sqrt{B}d_\star L_\star + \sqrt{B d_\star}S_\star L_\star  + \sqrt{B} M \right)\textcolor{DarkGreen}{(d_\star + \sqrt{d_\star}S_\star)L_\star\sqrt{  T} }
     + B\sum_{i \in \Ical} L_i S_i
     \right)\ .
\end{align*}
\end{theorem}
\begin{proof}
First, we consider the event where all learners with $d_i \geq d_\star$ are well-specified in the sense that their elliptical confidence intervals contain $\theta_\star$ at all times. This happens with probability at least $1 - M \delta$ by \pref{lem:yasin}. Further, only consider outcomes where \pref{thm:adv_epoch_balancing_regret} and \pref{thm:adv_epoch_balancing_termination} hold in all epochs.\footnote{We note that both theorems hold for arbitrary sequences of contexts and therefore also when the $s$-th instance of Epoch Balancing is started after a random number of rounds $t^{(s-1)}(T)$.}
By a union bound, all these assumptions hold with probability at least $1 - 4M$. 
We now consider the decomposition in Eq.~\eqref{equation::epoch_decomposition_regret} and bound
\begin{align*}
\regret(T) &= \sum_{s=1}^{M} \regret^{(s)}(t^{(s)}(T) - t^{(s-1)}(T))
\overset{(i)}{=} \sum_{s=1}^{i_\star} \regret^{(s)}(t^{(s)}(T) - t^{(s-1)}(T))
\\
& \overset{(ii)}{\leq} \sum_{s=1}^{i_\star} \left[C^{(s)} \sqrt{t^{(s)}(T) - t^{(s-1)}(T)} + 8.12\sum_{i \in \Ical^{(s)}} R^{\max}_i \ln \frac{5.2 M \ln (2T)}{\delta}\right]
\\
& \leq  \max_{s \in [i_\star]} C^{(s)} \sqrt{i_\star \sum_{s=1}^{i_\star} (t^{(s)}(T) - t^{(s-1)}(T))} + 8.12 i_\star \sum_{i \in \Ical^{(s)}} R^{\max}_i \ln \frac{5.2 M \ln (2T)}{\delta}
\\
& =  \max_{s \in [i_\star]} C^{(s)} \sqrt{i_\star  T} + 8.12 i_\star \sum_{i \in \Ical^{(s)}} R^{\max}_i \ln \frac{5.2 M \ln (2T)}{\delta}
\end{align*}
where $(i)$ follows from \pref{thm:adv_epoch_balancing_termination} and $(ii)$ from \pref{thm:adv_epoch_balancing_regret} with epoch-dependent factor $C^{(s)} \leq \tilde O\left(\left(d_\star + \sqrt{d_\star}S_\star  + M \right)(d_\star + \sqrt{d_\star}S_\star)\right)$ or  \pref{thm:adv_epoch_balancing_regret_ls} with epoch-dependent factor $C^{(s)} \leq \tilde O\left(\left(d_\star L_\star + \sqrt{d_\star}S_\star L_\star + M \right)(d_\star + \sqrt{d_\star}S_\star)\right) L_\star$
\end{proof}

%% file: conclusion.tex
\section{Conclusions}
We have described and analyzed a simple and general balancing and elimination technique to perform model selection in stochastic bandit and reinforcement learning tasks. We have instantiated our general principle to a number of relevant model selection scenarios with nested model classes, ranging from contextual linear bandits to linear MPDs, from mis-specified linear bandits and MDPs to hyperparameter tuning of the contextual bandit algorithm OFUL.
In all these cases, we show that the total regret of our master algorithm is bounded by the best valid candidate regret bound times a multiplicative factor. Notably, this factor becomes negligible in the presence of gaps in the regret bound guarantees across the base learners, so that in such cases we essentially recover the regret bound of the best base learner in hindsight.

Our work overcomes the limitations of previous approaches by combining ideas of a statistical test for arm elimination with regret balancing for exploration. We are able to obtain gap-dependent bounds, and go beyond the $\sqrt{MT}$ dependence of corralling methods based on adversarial master algorithms. 
The flexibility of our approach is also witnessed by our ability to extend the linear bandit analysis to the case of adversarial contexts by means of a randomized variant of our general balancing and elimination technique.

%% file: appendix/stochastic_proofs.tex
\section{Proofs for Setting with Stochastic Contexts} 
\label{app:stoch_proofs}
\begin{lemma}
\label{lem:mustar_conc}
There is an absolute constant $c$ such that the event \begin{align}
    \Gcal = \left\{ \forall i \in [M], \,\,\forall t \in \NN \colon |n_i(t) \mu^\star - U_i(t) - \regret_{i}(t)| \leq c\sqrt{\ln \frac{M \ln n_i(t)}{\delta} n_i(t)}\right\}
\end{align}
has probability at least $1 - \delta$
\end{lemma}
\begin{proof}
Consider a fixed $i \in [M]$ and write the LHS in the event definition as
\begin{align}
    &n_i(t) \mu^\star - U_i(t) - \regret_{i}(t)\\
    &= \sum_{k \in T_i(t)} \left(\mu^\star - r_k - \max_{\pi' \in \Pi} \EE[ r_k | \pi', x_k]  + \EE[ r_k | \pi_k, x_k]\right)\notag\\
    &= \sum_{k \in T_i(t)} \left(\mu^\star - \max_{\pi' \in \Pi} \EE[ r_k | \pi', x_k]\right)  + 
   \sum_{k \in T_i(t)} \left( \EE[ r_k | \pi_k, x_k]  - r_k\right).
   \label{eq:hp_proof1}
\end{align}
Consider the first sum and let $\Fcal_t$ be the sigma-field induced by all variables up to round $t$, i.e., $(\Ical_k, x_k, i_k, a_k, r_k)_{k \leq t}$. Note that $i_{t+1}$, the learner chosen at $t+1$ is $\Fcal_t$-measurable. Hence, $X_k = \one\{i_k = i\}(\mu^\star - \max_{\pi' \in \Pi} \EE[ r_k | \pi', x_k]) \in [-1, +1]$ is a martingale-difference sequence w.r.t. $\Fcal_k$. We will now apply a Hoeffding-style uniform concentration bound from \citet{howard2018uniform}.
Using the terminology and definition in this article, by case Hoeffding~I in Table~4, the process  $S_k = \sum_{j=1}^k X_k$ is sub-$\psi_N$ with variance process $V_k = \sum_{j=1}^k \one\{i_j = i\}/ 4$.
Thus by using the boundary choice in Equation~(11) of \citet{howard2018uniform}, we get
\begin{align*}
    S_k & \leq 1.7 \sqrt{V_k \left( \ln \ln(2 V_k) + 0.72 \ln(5.2 / \delta) \right) }\\
    & = 0.85\sqrt{n_i(k)\left( \ln \ln(n_i(k)/2) + 0.72 \ln(5.2 / \delta)\right)}
\end{align*}
for all $k$ where $V_k \geq 1$ with probability at least $1 - \delta$.
Applying the same argument to $-S_k$ gives that
\begin{align*}
    \left| \sum_{k \in T_i(t)} \left(\mu^\star - \max_{\pi' \in \Pi} \EE[ r_k | \pi', x_k]\right)\right|
    \leq 3 \vee 0.85\sqrt{n_i(k)\left( \ln \ln(n_i(k)/2) + 0.72 \ln(10.4 / \delta)\right)}
\end{align*}
holds with probability at least $1 - \delta$ for all $t$.

Consider now the second term in \eqref{eq:hp_proof1} and let $\Fcal_t$ now be the sigma-field induced by all variables up to the reward at round $t+1$, i.e., $\sigma((\Ical_k, x_k, i_k, a_k, r_k)_{k \leq t}, \Ical_{t+1}, x_{t+1}, i_{t+1}, a_{t+1})$. Then $X_k = \one\{i_k = i\}(\EE[ r_k | \pi_k, x_k]  - r_k) \in [-1, +1]$ is a martingale-difference sequence w.r.t. $\Fcal_k$ and we can apply the same concentration argument as for the first term to get  with probability at least $1 - \delta$ for all $t$
\begin{align*}
    \left| \sum_{k \in T_i(t)} \left( \EE[ r_k | \pi_k, x_k]  - r_k\right) \right|
    \leq 3 \vee 0.85\sqrt{n_i(k)\left( \ln \ln(n_i(k)/2) + 0.72 \ln(10.4 / \delta)\right)}~.
\end{align*}
We now take a union bound over both concentration results and $i \in [M]$ and rebind $\delta \rightarrow \delta / M$. Then picking the absolute constant $c$ sufficiently large gives the desired statement.
\end{proof}

\begin{lemma}[Sufficient Condition for Elimination]
\label{lem:suff_cond_elimination}
If the psuedo-regret of learner $i$ exceeds for any $\star \in \wellspidx$ the following bound in round $t$,
\begin{align}
    \regret_i(t) >&~
    \regretbound_i(n_i(t))
    + \frac{n_i(t)}{n_\star(t)} \regretbound_\star(n_\star(t))
  + 2c \left( 1 + \sqrt{\frac{n_i(t)}{n_\star(t)}}\right) \sqrt{n_i(t) \ln \frac{M \ln t}{\delta}} 
   \label{eqn:suff_cond_elimination_gen}
\end{align}
then learner $i$ fails the misspecification test of \pref{alg:balanced_elim_general} in event $\Gcal$ and is eliminated.
\end{lemma}
\begin{proof}
After dividing \pref{eqn:suff_cond_elimination_gen} by $n_i(t)$, this condition implies in event  $\Gcal_\star$ 
\begin{align*}
    \frac{\regret_i(t)}{n_i(t)} >&~
    \frac{\regretbound_i(n_i(t))}{n_i(t)}
    + \frac{\regret_\star(t)}{n_\star(t)}
    + 2c \sqrt{\frac{\ln(M \ln t / \delta)}{n_i(t)}} + 2c \sqrt{\frac{\ln(M \ln t / \delta)}{n_\star(t)}}
\end{align*}
and by $\Gcal$, this implies
\begin{align*}
    \mu_\star - \frac{U_i(t)}{n_i(t)} >~
    &~\frac{\regretbound_i(n_i(t))}{n_i(t)}
    + \mu_\star - \frac{U_\star(t)}{n_\star(t)}
    + c \sqrt{\frac{\ln(M \ln t / \delta)}{n_i(t)}} + c \sqrt{\frac{\ln(M \ln t / \delta)}{n_\star(t)}}~.
\end{align*}
Rearranging terms yields
\begin{align*}
      \frac{U_i(t)}{n_i(t)} &+ \frac{\regretbound_i(n_i(t))}{n_i(t)} 
      + c \sqrt{\frac{\ln(M \ln t / \delta)}{n_i(t)}}
      < 
     \frac{U_\star(t)}{n_\star(t)} - c \sqrt{\frac{\ln(M \ln t / \delta)}{n_\star(t)}}~.
\end{align*}
Hence, since $t > n_i(t)$ and $t > n_\star(t)$, the misspecification test in Algorithm~\ref{alg:balanced_elim_general} fails.
\end{proof}

\subsection{Special Case with $T^{\beta}$ Candidate Regret Bounds}

We here provide the proof of our gap-independent result which we restate here for convenience:
\polyregthm*
\begin{proof}
We start with the general regret bound from \pref{thm:general_regbound} given by
\begin{align}\label{eqn:general_regret_bound_app}
    \sum_{i=1}^M \regretbound_{\star}(n_{\star}(t_i))%
   + \sum_{i \in \misspidx} \frac{n_i(t_i)}{n_{\star}(t_i)} \regretbound_{\star}(n_{\star}(t_i)) + 2M
    + 
    2c \sum_{i \in \misspidx}  \left(1 + \sqrt{\frac{n_i(t_i)}{n_{\star}(t_i)}} \right)
    \sqrt{n_i(t_i) \ln\frac{M \ln T}{\delta}}~,
\end{align}
and bound the terms individually. We begin with
\begin{align*}
    \sum_{i=1}^M \regretbound_{\star}(n_{\star}(t_i)) + 2M
    &\leq M \regretbound_{\star}(T) + 2M 
    \leq M d_\star C T^{\beta} + 2M,
\end{align*}
where we only used the monotonicity of regret bounds and the definition of $\regretbound_\star$. We continue with the first part of the last term which we control as follows
\begin{align*}
    2c \sum_{i \in \misspidx} 
    \sqrt{n_i(t_i) \ln\frac{M \ln T}{\delta}}
    &\leq 2 c \sqrt{B \ln\frac{M\ln T}{\delta} \sum_{i \in \misspidx} n_i(t_i)} \leq 2 c \sqrt{B T \ln\frac{M \ln T}{\delta}}
\end{align*}
where we first applied Cauchy-Schwarz inequality and then used the fact that the total number of rounds played by all base learners is at most $T$. Similarly, we can bound the other part of the final term in \eqref{eqn:general_regret_bound_app} as
\begin{align*}
    2c \sum_{i \in \misspidx} \sqrt{\frac{n_i(t_i)}{n_{\star}(t_i)}}
    \sqrt{n_i(t_i) \ln\frac{M \ln T}{\delta}}
    &\leq 2 c \sqrt{\sum_{i \in \misspidx} \frac{n_i(t_i)}{n_{\star}(t_i)}} \sqrt{T \ln\frac{M \ln T}{\delta}}\\
    &\leq 
    2 \sqrt{2} c d_\star^{\frac{1}{2\beta}} \sqrt{B T \ln \frac{M \ln T}{\delta}}, 
\end{align*}
where the final step follows from \pref{lem:general_ratio_bound} with
\begin{align}
    \sum_{i \in \misspidx}  
 \frac{n_i(t_i)}{n_{\star}(t_i)} 
 \leq 2\sum_{i \in \misspidx}  \left(1 \vee \frac{d_\star^{1 / \beta}}{d_i^{1 / \beta}}  \right)
 \leq 2 B d_\star^{1 / \beta}~.
 \label{eqn:pull_ratio_proof1}
\end{align}
It only remains to bound the second term \eqref{eqn:general_regret_bound_app}. Here again we make use of the pull-ratio bound from \eqref{eqn:pull_ratio_proof1} to bound
\begin{align*}
    &\sum_{i \in \misspidx} \frac{n_i(t_i)}{n_{\star}(t_i)} \regretbound_{\star}(n_{\star}(t_i))
    =  C d_\star \sum_{i \in \misspidx} \left(\frac{n_i(t_i)}{n_\star(t_i)}\right)^{1 - \beta} n_i(t_i)^{\beta}\\
    & \leq 
    C d_\star \left( \sum_{i \in \misspidx} \frac{n_i(t_i)}{n_\star(t_i)} \right)^{1 - \beta}
    \left( \sum_{i \in \misspidx} n_i(t_i) \right)^{\beta}\\
    &  
    \leq C d_\star \left(2 B d_\star^{1 / \beta} \right)^{1 - \beta}
    T^{\beta} 
    \leq  
    2 C B^{1- \beta}  d_\star^{1 / \beta} 
    T^{\beta},
\end{align*}
where the first inequality follows from H\"older's inequality. Combining all bounds for the individual terms yields the desired statement.
\end{proof}

Below, we prove technical results for the slightly more general candidate regret bounds that can have different exponents $\beta$. Specifically, we consider candidate regret bounds of the form
\begin{align}
    \regretbound_i(n) = n \wedge C d_i n^{\beta_i},
    \label{eqn:general_poly_candidate}
\end{align}
where $\beta_i \in (0, 1]$, $d_i \geq 1$ and $C$ is a term that does not depend on $i$ or $n$.

\begin{lemma}[Play ratio bound]\label{lem:general_ratio_bound}
If \pref{alg:balanced_elim_general} is used with candidate regret bounds of the form in Equation~\eqref{eqn:general_poly_candidate}, then
\begin{align*}
    \frac{n_i(t)}{n_j(t)} \leq \begin{cases}
    \left(2\frac{d_j}{d_i}\right)^{\frac{1}{\beta_i}}  n_j(t)^{\frac{\beta_j}{\beta_i} - 1}
    & \textrm{if } n_i(t) \geq (d_i C)^{\frac{1}{1 - \beta}}\\
    2 &  \textrm{if } n_i(t) \leq (d_i C)^{\frac{1}{1 - \beta}}
    \end{cases}
\end{align*}
holds for all $t$ and active learners $i,j \in \Ical_t$ that have been played at least once.
\end{lemma}
\begin{proof}
By \pref{lem:balancing}, the regret bound of $i$ and $j$ are balanced at $t$, which means that
\begin{align*}
    \regretbound_i(n_i(t)) \leq \regretbound_{j}(n_{j}(t)) + 1 \leq 2\regretbound_{j}(n_{j}(t))~.
\end{align*}
When $n_i(t) \leq (d_i C)^{\frac{1}{1 - \beta}}$ the regret bound $\regretbound_i$ is still in the linear regime. The balancing condition gives in this case 
$n_i(t) \leq 2\regretbound_{j}(n_{j}(t))  \leq 2 n_j(t) $ and hence
$\frac{n_i(t)}{n_j(t)} \leq 2$.
Consider now the case where $\regretbound_i$ is in the $n_i(t)^{\beta_i}$ regime. Then the balancing condition implies
\begin{align*}
    d_i C n_i(t)^{\beta_i} \leq 2d_j C n_j(t)^{\beta_j} .
\end{align*}
Reordering terms yields
\begin{align*}
    \left(\frac{n_i(t)}{n_j(t)}\right)^{\beta_i} 
    \leq 2\frac{d_j}{d_i}  n_j(t)^{\beta_j - \beta_i}~.
\end{align*}
\end{proof}

\paragraph{Gap-dependent guarantee:} 
We now provide the full proof for our main gap-dependent guarantee which we restate her for convenience: 

\polyreggapthm*

\begin{proof}
Just as for the gap-independent guarantee in \pref{thm:poly_reg}, we start with the general regret bound from \pref{thm:general_regbound} given by
\begin{align}\label{eqn:general_regret_bound_app2}
    \sum_{i=1}^M \regretbound_{\star}(n_{\star}(t_i))%
   + \sum_{i \in \misspidx} \frac{n_i(t_i)}{n_{\star}(t_i)} \regretbound_{\star}(n_{\star}(t_i)) + 2M
    + 
    2c \sum_{i \in \misspidx}  \left(1 + \sqrt{\frac{n_i(t_i)}{n_{\star}(t_i)}} \right)
    \sqrt{n_i(t_i) \ln\frac{M \ln T}{\delta}}~,
\end{align}
and bound the terms individually. We begin with
\begin{align*}
    \sum_{i=1}^M \regretbound_{\star}(n_{\star}(t_i)) + 2M
    &\leq M \regretbound_{\star}(T) + 2M 
    \leq M d_\star C T^{\beta} + 2M,
\end{align*}
where we only used the monotonicity of regret bounds and the definition of $\regretbound_\star$. 
All remaining terms only consider misspcified learners $i \in \misspidx$. In the following, we bound the contribution from each such learner individually. We have
\begin{align}
\label{eqn:reg_mis_proof11}
    &\frac{n_i(t_i)}{n_{\star}(t_i)} \regretbound_{\star}(n_{\star}(t_i)) + \left(1 + \sqrt{\frac{n_i(t_i)}{n_{\star}(t_i)}} \right)
    \sqrt{n_i(t_i) \ln\frac{M \ln T}{\delta}}\notag\\
    &\leq C d_\star \left(\frac{n_i(t_i)}{n_{\star}(t_i)}\right)^{1 - \beta} n_i(t_i)^\beta + \left(1 + \sqrt{\frac{n_i(t_i)}{n_{\star}(t_i)}} \right)
    \sqrt{n_i(t_i) \ln\frac{M \ln T}{\delta}}\notag\\
    & \leq C d_\star Z^{1 - \beta} n_i(t_i)^\beta + \left(1 + \sqrt{Z} \right)
    \sqrt{n_i(t_i) \ln\frac{M \ln T}{\delta}}\notag\\
    & \leq C d_\star Z^{1 - \beta} n_i(t_i)^\beta + 2
    \sqrt{Z n_i(t_i) \ln\frac{M \ln T}{\delta}},
\end{align}
where $Z = 2 \vee \left(2\frac{d_\star}{d_i}\right)^{\frac{1}{\beta}}$. Further, 
using the gap-assumption, \pref{lem:elimination_gen_obs}, which is proved below, yields an upper-bound on the number of times the learner can be played
\begin{align*}
        n_i(T) 
   & \leq
    \left[ \frac{2C d_i}{\Delta_i} \left( 1 + 2 Z\right)
    \right]^{\frac{1}{\alpha - \beta}} \vee \left[\frac{4c}{\Delta_i} \left( 1 + \sqrt{Z}
    \right) \sqrt{\ln \frac{M \ln T}{\delta}}\right]^{\frac{1}{\alpha - 1/2}}
\\    &\leq
    \left[ \frac{5C d_i}{\Delta_i}  Z
    \right]^{\frac{1}{\alpha - \beta}} \vee \left[\frac{8c}{\Delta_i} \sqrt{Z}
    \sqrt{\ln \frac{M \ln T}{\delta}}\right]^{\frac{1}{\alpha - 1/2}}~.
\end{align*}
We consider now two cases. 
\paragraph{Case I: $\beta \geq 1/2$.} Then $n_i(T) \leq
    \left[ \frac{5 C d_i}{\Delta_i}  Z \right]^{\frac{1}{\alpha - \beta}}$
    and \eqref{eqn:reg_mis_proof11} can be bounded as
\begin{align*}
   & C d_\star Z^{1 - \beta} n_i(t_i)^\beta + 2
    \sqrt{Zn_i(t_i) \ln\frac{M \ln T}{\delta}}
    \leq 3C \sqrt{ \ln\frac{M \ln T}{\delta}}  d_\star Z^{1 - \beta} n_i(t_i)^\beta\\
    & \leq 3C \sqrt{ \ln\frac{M \ln T}{\delta}}  d_\star Z^{1 - \beta} \left[ \frac{5C d_i}{\Delta_i}  Z \right]^{\frac{\beta}{\alpha - \beta}}~.
\end{align*}
When $Z = 2$, then this expression is bounded from above as $6C \sqrt{ \ln\frac{M \ln T}{\delta}}  d_\star  \left[ \frac{10C d_i}{\Delta_i} \right]^{\frac{\beta}{\alpha - \beta}}$.
When $Z > 2$, then we bound this quantity instead as
\begin{align*}
   &3C \sqrt{ \ln\frac{M \ln T}{\delta}}  d_\star (2d_\star)^{\frac{1 - \beta}{\beta}} \left[ \frac{5C d_i}{\Delta_i}  \left(\frac{2d_\star}{d_i}\right)^{1 / \beta}\right]^{\frac{\beta}{\alpha - \beta}}
    \leq
    6C \sqrt{ \ln\frac{M \ln T}{\delta}}  d_\star^{\frac{1}{\beta} + \frac{1}{\alpha - \beta}} \left[ \frac{20C}{\Delta_i} \right]^{\frac{\beta}{\alpha - \beta}}.
\end{align*}
Hence, the total regret is bounded is case as
\begin{align*}
    \regret(T) = O\left( M d_\star C T^{\beta} + \sum_{i \in \misspidx} 
    C \sqrt{ \ln\frac{M \ln T}{\delta}}  \left( d_\star^{\frac{1}{\beta} + \frac{1}{\alpha - \beta}} + d_\star d_i^{\frac{\beta}{\alpha - \beta}}
    \right) \left[ \frac{20C}{\Delta_i} \right]^{\frac{\beta}{\alpha - \beta}}
    \right).
\end{align*}
\paragraph{Case II: $\beta < 1/2$.}
To simplify the final bound, we here use the somewhat crude bound on $n_i(T)$:
\begin{align*}
        n_i(T) 
   & \leq \left[ \frac{5C d_i}{\Delta_i}  Z\sqrt{\ln \frac{M \ln T}{\delta}}
    \right]^{\frac{1}{\alpha - 1/2}}
\end{align*}
This allows us to upper-bound \eqref{eqn:reg_mis_proof11} by
\begin{align*}
3Cd_\star Z^{1 - \beta}
    \sqrt{n_i(t_i) \ln\frac{M \ln T}{\delta}}
    \leq 
    3Cd_\star Z^{1 - \beta}
    \sqrt{\ln\frac{M \ln T}{\delta}}\left[ \frac{5C d_i}{\Delta_i}  Z\sqrt{\ln \frac{M \ln T}{\delta}}
    \right]^{\frac{1/2}{\alpha - 1/2}}.
\end{align*}
When $Z = 2$, this expression is bounded from above by
$6Cd_\star 
    \left[ \frac{10C d_i}{\Delta_i} \ln \frac{M \ln T}{\delta}
    \right]^{\frac{1/2}{\alpha - 1/2}}$.
When $Z > 2$, then we bound this quantity instead as
\begin{align*}
    &3Cd_\star (2d_\star)^{\frac{1 - \beta}{\beta}}
    \sqrt{\ln\frac{M \ln T}{\delta}}\left[ \frac{5C d_i}{\Delta_i}  \left(\frac{2d_\star}{d_i}\right)^{1 / \beta}\sqrt{\ln \frac{M \ln T}{\delta}}
    \right]^{\frac{1/2}{\alpha - 1/2}}
    \\&\leq 
    2C (2d_\star)^{\frac{1}{\beta}}
    \left[ \frac{5C}{\Delta_i}  \left(2d_\star\right)^{1 / \beta}\ln \frac{M \ln T}{\delta}
    \right]^{\frac{1/2}{\alpha - 1/2}}
\end{align*}
Hence, the total regret is bounded is case as
\begin{align*}
    \regret(T) = O\left( M d_\star C T^{\beta} + \sum_{i \in \misspidx} 
    C   \left( (2d_\star)^{\frac{1}{\beta} + \frac{1}{\beta(2\alpha - 1)}} + d_\star d_i^{\frac{1}{2\alpha - 1}}
    \right) \left[ \frac{20C }{\Delta_i} \ln \frac{M \ln T}{\delta}\right]^{\frac{1}{2\alpha - 1}}
    \right).
\end{align*}
\end{proof}

\begin{lemma}[Gap-dependent elimination bound]
\label{lem:elimination_gen_obs}
Assume \pref{alg:balanced_elim_general} is used with candidate regret bound of the form in Equation~\eqref{eqn:general_poly_candidate}. If the pseudo-regret of base-learner $i$ satisfies $\regret_i(t) \geq \Delta_i n_i(t)^{\alpha_i}$ for all $t$ for a fixed $\Delta_i >0$ and $\alpha_i > \frac 1 2 \vee \beta_i$, then, in event $\Gcal$, learner $i$ is played at most
\begin{align*}
    n_i(T) 
    \leq
    \left[ \frac{2C d_i}{\Delta_i} \left( 1 + 2 Z\right)
    \right]^{\frac{1}{\alpha_i - \beta_i}} \vee \left[\frac{4c}{\Delta_i} \left( 1 + \sqrt{Z}
    \right) \sqrt{\ln \frac{M \ln T}{\delta}}\right]^{\frac{1}{\alpha_i - 1/2}},
\end{align*}
times where $Z = 2 \vee \left(2\frac{d_\star}{d_i}\right)^{\frac{1}{\beta_i}}  n_\star(t_i)^{\frac{\beta_\star}{\beta_i} - 1}$ and $\star \in \wellspidx$ is any well-specified learner.
\end{lemma}
\begin{proof}
\pref{lem:suff_cond_elimination} yields the following sufficient condition that learner $i$ is eliminated at round $t$:
\begin{align}
    \regret_i(t) >&~
    \regretbound_i(n_i(t))
    + \frac{n_i(t)}{n_\star(t)} \regretbound_\star(n_\star(t))
    + 2c \left( 1 + \sqrt{\frac{n_i(t)}{n_\star(t)}}\right) \sqrt{n_i(t) \ln \frac{M \ln t}{\delta}}.
    \label{eqn:suff_cond_elimination_gen_app1}
\end{align}
We now upper-bound the RHS of this sufficient condition using \pref{lem:general_ratio_bound} as
\begin{align*}
      \regretbound_i(n_i(t)) &
    + \frac{n_i(t)}{n_\star(t)}\regretbound_\star(n_\star(t))
    + 2c \left( 1 + \sqrt{\frac{n_i(t)}{n_\star(t)}}\right) \sqrt{n_i(t) \ln \frac{M \ln t}{\delta}}
    \\
          \leq & \regretbound_i(n_i(t))
    + 2\frac{n_i(t)}{n_\star(t)}\regretbound_i(n_i(t))
    + 2c \left( 1 + \sqrt{\frac{n_i(t)}{n_\star(t)}}\right) \sqrt{n_i(t) \ln \frac{M \ln t}{\delta}}
    \\
              \leq & 
    \left( 1 + 2 Z \right)
    \regretbound_i(n_i(t))
    + 2c \left( 1 + \sqrt{Z}
    \right) \sqrt{n_i(t) \ln \frac{M \ln t}{\delta}}\nonumber
        \\
              \leq & 
    \left( 1 + 2 Z\right)
    C d_i n_i(t)^{\beta_i}
    + 2c \left( 1 +  \sqrt{Z}
    \right) \sqrt{n_i(t) \ln \frac{M \ln t}{\delta}}.\nonumber
\end{align*}
Using this upper-bound on the RHS of \eqref{eqn:suff_cond_elimination_gen_app1} and $\Delta_i n_i(t)^{\alpha_i}$ as a lower-bound on the LHS of \eqref{eqn:suff_cond_elimination_gen_app1}, we can conclude that learner $i$ gets eliminated if the following two conditions are met:
\begin{align*}
 \frac{\Delta_i}{2} n_i(t)^{\alpha_i} &> 2c \left( 1 + \sqrt{Z}
    \right) \sqrt{n_i(t) \ln \frac{M \ln t}{\delta}}\\
     \frac{\Delta_i}{2} n_i(t)^{\alpha_i} &> \left( 1 + 2 Z\right)
    C d_i n_i(t)^{\beta_i}
\end{align*}
Rearranging each condition yields
\begin{align*}
n_i(t) &> \left[\frac{4c}{\Delta_i} \left( 1 + \sqrt{Z}
    \right) \sqrt{\ln \frac{M \ln t}{\delta}}\right]^{\frac{1}{\alpha_i - 1/2}} \quad \textrm{and} \quad
    n_i(t) >
    \left[ \frac{2 C d_i}{\Delta_i} \left( 1 + 2Z\right)
    \right]^{\frac{1}{\alpha_i - \beta_i}}.
\end{align*}
\end{proof}

\subsection{Special Case with $\sqrt{T \ln T}$ Candidate Regret Bounds}

Consider the regret bound for all $M$ base learners to be of the form
\begin{align}
\label{eqn:sqrt_reg_candidates}
    \regretbound_i(n) = d_i C \sqrt{n \ln_+(n / \delta))} \wedge n
\end{align}
where $\ln_+(x) = \ln(x \vee e)$ and $d_i \geq 1$ is some parameter (not necessarily an integer dimension) and $C \geq 1$ is some term that does not depend on $n$ or $i$.
To prepare for proving the main regret guarantee, we first show a bound on the play ratio between two active learners:

\begin{lemma}\label{lem:sqrt_ratio_bound}
For the choice of candidate regret bounds in Equation~\eqref{eqn:sqrt_reg_candidates}, the following bound
\begin{align*}
    \frac{n_i(t)}{n_{j}(t)} \leq  7 \left( 1 \vee \frac{d_j^2}{d_i^2}\right)\ln_+\left( 4e   \ln\frac{t}{\delta}\right)
\end{align*}
holds for all $t$ and active learners $i,j \in \Ical_{t+1}$ that have been played at least once.
\end{lemma}
\begin{proof}
By \pref{lem:balancing}, the regret bound of $i$ and $j$ are balanced at $t$, which means that
\begin{align*}
    \regretbound_i(n_i(t)) \leq \regretbound_{j}(n_{j}(t)) + 1 \leq 2 \regretbound_{j}(n_{j}(t))~.
\end{align*}
When $\regretbound_i$ is still in the linear regime, this implies that 
$n_i(t) \leq \regretbound_{j}(n_{j}(t)) + 1 \leq n_j(t_i) + 1$ and hence
$\frac{n_i(t)}{n_j(t)} \leq 2$.
Consider now the case where $\regretbound_i$ is in the $\sqrt{\cdot}$~-regime. Then the balancing condition implies
\begin{align*}
    d_i C\sqrt{n_i(t) \ln_+\frac{n_i(t)}{\delta})} \leq 2 d_j C\sqrt{n_j(t) \ln_+\frac{n_j(t)}{\delta}} 
\end{align*}
and thus
\begin{align*}
   &\sqrt{ \frac{n_i(t) \ln_+(n_i(t)/\delta)}{n_j(t) \ln_+(n_j(t)/\delta)}} \leq 2 \frac{d_j}{d_i} .
\end{align*}
Reordering this inequality gives:
\begin{align}
    \frac{n_i(t)}{n_j(t)} \leq 4 \frac{d_j^2}{d_i^2} \frac{\ln_+(n_j(t)/\delta)}{\ln_+(n_i(t)/\delta)}
    \leq 4 \frac{d_j^2}{d_i^2} \ln_+(n_j(t) / \delta)
    \leq 4 \frac{d_j^2}{d_i^2} \ln(t / \delta)~.
    \label{eqn:ratio_bound_proof_sqrt1}
\end{align}
 We now refine this crude bound by considering two cases: 
 \paragraph{Case I:} If $\sqrt{n_j(t)} \leq C d_j \sqrt{\ln_+(n_j(t)/ \delta)}$, then $\regretbound_j(n_j(t)_ = n_j(t)$ and the balancing condition gives $n_j(t) \leq 2n_i(t)$ Plugging this in \eqref{eqn:ratio_bound_proof_sqrt1} yields
 \begin{align*}
      \frac{n_i(t)}{n_j(t)} \leq 4 \frac{d_j^2}{d_i^2} \frac{\ln_+(2n_i(t)/\delta)}{\ln_+(n_i(t)/\delta)}
      \leq 4 \frac{d_j^2}{d_i^2} \ln (2e)
      \leq 7 \frac{d_j^2}{d_i^2}.
 \end{align*}
 \paragraph{Case II:} In this case, $\regretbound_j(n_j(t)) = C d_j \sqrt{n_j(t)}$ and we use \eqref{eqn:ratio_bound_proof_sqrt1} with reversed roles of $i,j$ to get $n_j(t) \leq 4 \frac{d_i^2}{d_j^2} \ln(t / \delta) n_i(t)$. Plugging this back into the middle term of \eqref{eqn:ratio_bound_proof_sqrt1} yields
\begin{align*}
\frac{n_i(t)}{n_j(t)} &\leq 4 \frac{d_j^2}{d_i^2} \ln_+(e4 d_i^2 / d_j^2 \ln(t / \delta)).
\end{align*}
When $d_j^2 / d_i^2 \geq 1$, then $\frac{n_i(t)}{n_j(t)} \leq 4 \frac{d_j^2}{d_i^2} \ln_+(e4 \ln(t / \delta))$ follows immediately. Otherwise,
\begin{align*}
\frac{n_i(t)}{n_j(t)} &\leq 4 \frac{d_j^2}{d_i^2} \ln_+(e4 d_i^2 / d_j^2 \ln(t / \delta))
\leq 
4 \frac{d_j^2}{d_i^2} \ln(d_i^2 / d_j^2) + 4 \frac{d_j^2}{d_i^2} \ln(e4 \ln(t / \delta))\\
&\leq \frac{4}{e}  + 4 \ln(e4 \ln(t / \delta))
\leq 4 \ln(4 \ln(t / \delta))
\end{align*}
\end{proof}

\begin{theorem}\label{thm:sqrt_reg}
If \pref{alg:balanced_elim_general} is used with candidate regret bounds in Equation~\eqref{eqn:sqrt_reg_candidates}, then its total regret is bounded with probability at least $1- \delta$ for all $T$ as
\begin{align*}
        \regret(T) &
        \leq\left(M + d_\star \sqrt{B \ln_+\left(11 \ln \frac{T}{\delta}\right)}\right) d_\star C\sqrt{T \ln_+(T / \delta)} + 2M\\
        &\quad+ 8c d_\star  \ln \left(\frac{11 M \ln T}{\delta} \right)  \sqrt{B T}
\end{align*}
where $\star \in \Wcal$ is any well-specified learner and $B = |\misspidx|$ is the number of misspecified learners.
\end{theorem}
\begin{proof}
We start with the general regret bound from \pref{thm:general_regbound} given by
\begin{align}\label{eqn:general_regret_bound_app3}
    \sum_{i=1}^M \regretbound_{\star}(n_{\star}(t_i))%
   + \sum_{i \in \misspidx} \frac{n_i(t_i)}{n_{\star}(t_i)} \regretbound_{\star}(n_{\star}(t_i)) + 2M
    + 
    2c \sum_{i \in \misspidx}  \left(1 + \sqrt{\frac{n_i(t_i)}{n_{\star}(t_i)}} \right)
    \sqrt{n_i(t_i) \ln\frac{M \ln T}{\delta}}~,
\end{align}
and bound the terms individually. We begin with
\begin{align*}
    \sum_{i=1}^M \regretbound_{\star}(n_{\star}(t_i)) + 2M
    &\leq M \regretbound_{\star}(T) + 2M 
    \leq M d_\star C \sqrt{T \ln_+(T / \delta)} + 2M,
\end{align*}
where we only used the monotonicity of regret bounds and the definition of $\regretbound_\star$. We continue with the first part of the last term which we control as follows
\begin{align*}
    2c \sum_{i \in \misspidx} 
    \sqrt{n_i(t_i) \ln\frac{M \ln T}{\delta}}
    &\leq 2 c \sqrt{B \ln\frac{M\ln T}{\delta} \sum_{i \in \misspidx} n_i(t_i)} \leq 2 c \sqrt{B T \ln\frac{M \ln T}{\delta}}
\end{align*}
where we first applied Cauchy-Schwarz inequality and then used the fact that the total number of rounds played by all base learners is at most $T$. Similarly, we can bound the other part of the final term in \eqref{eqn:general_regret_bound_app3} as
\begin{align*}
    2c \sum_{i \in \misspidx} \sqrt{\frac{n_i(t_i)}{n_{\star}(t_i)}}
    \sqrt{n_i(t_i) \ln\frac{M \ln T}{\delta}}
    &\leq 2 c \sqrt{\sum_{i \in \misspidx} \frac{n_i(t_i)}{n_{\star}(t_i)}} \sqrt{T \ln\frac{M \ln T}{\delta}}\\
    &\leq 
    6c \sqrt{ B \ln_+\left( 4e   \ln\frac{T}{\delta}\right) }  d_\star \sqrt{ T \ln \frac{M \ln T}{\delta}}, 
\end{align*}
where the final step follows from \pref{lem:sqrt_ratio_bound} with
\begin{align}
    \sum_{i \in \misspidx}  
 \frac{n_i(t_i)}{n_{\star}(t_i)} 
 \leq 7\sum_{i \in \misspidx} 
  \left( 1 \vee \frac{d_\star^2}{d_i^2}\right)\ln_+\left( 4e   \ln\frac{t_i}{\delta}\right) 
  \leq 7 d_\star^2 B \ln_+\left( 4e   \ln\frac{T}{\delta}\right) 
 \label{eqn:pull_ratio_proof12}
\end{align}
It only remains to bound the second term \eqref{eqn:general_regret_bound_app3}. Here again we make use of the pull-ratio bound from \eqref{eqn:pull_ratio_proof12} to bound
\begin{align*}
    &\sum_{i \in \misspidx} \frac{n_i(t_i)}{n_{\star}(t_i)} \regretbound_{\star}(n_{\star}(t_i))
    =  C d_\star \sum_{i \in \misspidx} \left(\frac{n_i(t_i)}{n_\star(t_i)}\right)^{1/2} n_i(t_i)^{1/2} \sqrt{\ln_+(n_\star(t_i) / \delta)}\\
    & \leq 
    C d_\star \sqrt{\ln_+(T / \delta)} \sqrt{ \sum_{i \in \misspidx} \frac{n_i(t_i)}{n_\star(t_i)}}
    \sqrt{ \sum_{i \in \misspidx} n_i(t_i) }
    \leq 3 C d_\star^2 \sqrt{B T \ln_+(T / \delta) \ln_+\left( 4e   \ln\frac{T}{\delta}\right) },
\end{align*}
where the first inequality follows from the Cauchy-Schwarz inequality. Combining all bounds for the individual terms yields the desired statement.
\end{proof}

\paragraph{Gap-dependent Regret Guarantee:} We now prove a gap-dependent regret bound for \pref{alg:balanced_elim_general} when used with candidate regret bounds in Equation~\eqref{eqn:sqrt_reg_candidates}.
\begin{lemma}
[Gap-dependent elimination bound]
\label{lem:elimination_sqrt_obs}
Assume \pref{alg:balanced_elim_general} is used with candidate regret bound of the form in Equation~\eqref{eqn:sqrt_reg_candidates}. If the pseudo-regret of base-learner $i$ satisfies $\regret_i(t) \geq \Delta_i n_i(t)^{\alpha_i}$ for all $t$ for a fixed $\Delta_i >0$ and $\alpha_i > \frac 1 2$, then, in event $\Gcal$, learner $i$ is played at most
\begin{align*}
    n_i(T) 
    \leq
\left[\frac{2C d_i}{\Delta_i} \left( 1 + 2 Z\right)\sqrt{\ln_+(MT / \delta)} \right]^{\frac{1}{\alpha_i - 1/2}},
\end{align*}
times where $Z = 7 \left( 1 \vee \frac{d_j^2}{d_i^2}\right)\ln_+\left( 4e   \ln\frac{t}{\delta}\right)$ and $\star \in \wellspidx$ is any well-specified learner.
\end{lemma}
\begin{proof}
This statement can be proved in full analogy to \pref{lem:elimination_gen_obs}.
\end{proof}

\begin{theorem}
\label{thm:sqrt_reg_gap}
Assume \pref{alg:balanced_elim_general} is used with candidate regret bounds in Equation~\eqref{eqn:sqrt_reg_candidates} and that the pseudo-regret of all misspcified learners $j \in \misspidx$ is bounded for all $t$ from below as $\regret_j(t) \geq \Delta_j n_j(t)^{\alpha}$ for some $\alpha > \frac{1}{2} \vee \beta$ and $\Delta_j > 0$. Then total regret is bounded with probability at least $1- \delta$ for all $T$ as
\begin{align}
\regret(T) &\leq M d_\star C \sqrt{T \ln_+(T / \delta)} + 2M \\
&\!\!\!\!+ 
    9C d_\star  \sum_{i \in \misspidx}  \ln_+\left( 4e   \ln\frac{T}{\delta}\right)^{\frac{1}{2} + \frac{1}{2\alpha - 1}}
    \left(\ln_+ \frac{MT}{\delta} \right)^{\frac{1}{2} + \frac{1/2}{2\alpha - 1}} 
    \left[\frac{42 d_i C}{\Delta_i}   \right]^{\frac{1}{2\alpha - 1}}
    \left( 1 \vee \frac{d_\star}{d_i}\right)^{1 + \frac{2}{2\alpha - 1}}.\nonumber
\end{align}
for $\star \in \Wcal$ is any well-specified learner.
\end{theorem}

\begin{proof}
Just as for the gap-independent guarantee in \pref{thm:sqrt_reg}, we start with the general regret bound from \pref{thm:general_regbound} given by
\begin{align*}%
    \sum_{i=1}^M \regretbound_{\star}(n_{\star}(t_i))%
   + \sum_{i \in \misspidx} \frac{n_i(t_i)}{n_{\star}(t_i)} \regretbound_{\star}(n_{\star}(t_i)) + 2M
    + 
    2c \sum_{i \in \misspidx}  \left(1 + \sqrt{\frac{n_i(t_i)}{n_{\star}(t_i)}} \right)
    \sqrt{n_i(t_i) \ln\frac{M \ln T}{\delta}}~,\nonumber
\end{align*}
and bound the terms individually. We begin with
\begin{align*}
    \sum_{i=1}^M \regretbound_{\star}(n_{\star}(t_i)) + 2M
    &\leq M \regretbound_{\star}(T) + 2M 
    \leq M d_\star C \sqrt{T \ln_+(T / \delta)} + 2M,
\end{align*}
where we only used the monotonicity of regret bounds and the definition of $\regretbound_\star$. 
All remaining terms only consider misspcified learners $i \in \misspidx$. In the following, we bound the contribution from each such learner individually. We have
\begin{align*}
    &\frac{n_i(t_i)}{n_{\star}(t_i)} \regretbound_{\star}(n_{\star}(t_i)) + \left(1 + \sqrt{\frac{n_i(t_i)}{n_{\star}(t_i)}} \right)
    \sqrt{n_i(t_i) \ln\frac{M \ln T}{\delta}}\\
    &\leq C d_\star \sqrt{\frac{n_i(t_i)}{n_{\star}(t_i)}}
    \sqrt{n_i(t_i) \ln_+(n_\star(t_i) / \delta)} + \left(1 + \sqrt{\frac{n_i(t_i)}{n_{\star}(t_i)}} \right)
    \sqrt{n_i(t_i) \ln\frac{M \ln T}{\delta}}
    \notag
    \\
    & \leq C d_\star \sqrt{Z n_i(t_i) \ln_+(T / \delta)} + \left(1 + \sqrt{Z} \right)
    \sqrt{n_i(t_i) \ln\frac{M \ln T}{\delta}}\notag\\
        & \leq C d_\star \sqrt{Z n_i(t_i) \ln_+ \frac{T}{\delta}} + 2\sqrt{Z} 
    \sqrt{n_i(t_i) \ln\frac{M \ln T}{\delta}}\notag\\
        & \leq 3C d_\star \sqrt{Z n_i(t_i) \ln_+ \frac{MT}{\delta} }\notag
\end{align*}
where $Z  = 7 \left( 1 \vee \frac{d_\star^2}{d_i^2}\right)\ln_+\left( 4e   \ln\frac{T}{\delta}\right)$. Further, 
using the gap-assumption, \pref{lem:elimination_sqrt_obs} yields an upper-bound on the number of times the learner can be played
\begin{align*}
    n_i(T) 
    & \leq
\left[\frac{2C d_i}{\Delta_i} \left( 1 + 2 Z\right)\sqrt{\ln_+(MT / \delta)} \right]^{\frac{1}{\alpha - 1/2}}
    \leq
\left[\frac{6Z C d_i}{\Delta_i} \sqrt{\ln_+(MT / \delta)} \right]^{\frac{1}{\alpha - 1/2}}\\
& \leq \left[\frac{42 C}{\Delta_i} \left( d_i \vee \frac{d_\star^2}{d_i}\right)\ln_+\left( 4e   \ln\frac{T}{\delta}\right) \sqrt{\ln_+(MT / \delta)} \right]^{\frac{1}{\alpha - 1/2}}
\end{align*}
We use this upper-bound to control the term
\begin{align*}
    &3C d_\star \sqrt{Z n_i(t_i) \ln_+ \frac{MT}{\delta} }\\
    &\leq 
    9C d_\star \left(1 \vee \frac{d_\star}{d_i}\right) \ln_+\left( 4e   \ln\frac{T}{\delta}\right)^{\frac{1}{2} + \frac{1}{2\alpha - 1}}
    \left(\ln_+ \frac{MT}{\delta} \right)^{\frac{1}{2} + \frac{1/2}{2\alpha - 1}} 
    \left[\frac{42 C}{\Delta_i} \left( d_i \vee \frac{d_\star^2}{d_i}\right)  \right]^{\frac{1}{2\alpha - 1}}.
\end{align*}
Combining all bounds of individual terms yields the desired bound
\begin{align*}
\regret(T) &\leq M d_\star C \sqrt{T \ln_+(T / \delta)} + 2M \\
&\qquad+ 
    9C d_\star  \sum_{i \in \misspidx}  \ln_+\left( 4e   \ln\frac{T}{\delta}\right)^{\frac{1}{2} + \frac{1}{2\alpha - 1}}
    \left(\ln_+ \frac{MT}{\delta} \right)^{\frac{1}{2} + \frac{1/2}{2\alpha - 1}} 
    \left[\frac{42 d_i C}{\Delta_i}   \right]^{\frac{1}{2\alpha - 1}}
    \left( 1 \vee \frac{d_\star}{d_i}\right)^{1 + \frac{2}{2\alpha - 1}}
\end{align*}
\end{proof}

\subsection{Special Case with $\epsilon_i C_2 T + C_1 \sqrt{T}$ Candidate Regret Bounds}

\begin{lemma}
\label{lem:lin_sqrt_regret_bound}
Assume all base algorithms use regret bounds of the form \eqref{eqn:linT_regcand} in \pref{thm:linTmissp}. Let $i \in \Ical_{t+1}$ be an active learner and $\bestallwell \in \wellspidx$ be a well-specified learner with $\epsilon_\bestallwell \geq \epsilon_i$. Then in event $\Gcal$
\begin{align*}
    \regret_i(t) &\leq 1 + 10 \regretbound_\bestallwell(n_\bestallwell(t)) + 2 \epsilon_\bestallwell C_2 \left(1 + \frac{c}{C_1} \sqrt{\ln \frac{M \ln t}{\delta}} \right) n_i(t)
    \\
    &\qquad + 8c \sqrt{n_i(t) \ln \frac{ M \ln t}{\delta}} 
    + 8 C_1^2 + 2 C_1 \sqrt{n_i(t)} + 8 c C_1 \sqrt{\ln \frac{M \ln t}{\delta}}~. 
\end{align*}
\end{lemma}
\begin{proof}
First, we can assume without loss of generality that $C_2 \epsilon_\bestallwell \leq 1$ because the regret bound is vacuous otherwise.
Since $i$ is in the active set and $\bestallwell$ is well-specified, we can apply \pref{lem:regbound_nonelim} which gives
\begin{align}\label{eqn:basereg_before_case}
\regret_i(t) \leq 1 + \regretbound_\bestallwell(n_\bestallwell(t)) + 2 c \sqrt{n_i(t) \ln \frac{M \ln t}{\delta}} 
+ \frac{n_i(t)}{n_\bestallwell(t)} \regretbound_\bestallwell(n_\bestallwell(t)) + 2 c \sqrt{\frac{n_i(t)^2}{n_\bestallwell(t)} \ln \frac{M \ln t}{\delta}}~.
\end{align}
We now simplify the expression on the right hand side using the specific form of the regret bounds $\regretbound_j$. This form can be split into three phases:
\begin{align*}
    \regretbound_j(n) &= n & \textrm{for } \sqrt{n}& \leq \frac{C_1}{1 - C_2 \epsilon_j} & \textrm{Phase I}\\
    \regretbound_j(n) &\in [C_1 \sqrt{n}, 2 C_1 \sqrt{n}] & \textrm{for } \frac{C_1}{1 - C_2 \epsilon_j} & < \sqrt{n}
    \leq \frac{C_1}{C_2 \epsilon_j} & \textrm{Phase II}\\
        \regretbound_j(n) &\in [C_2 \epsilon_j n, 2 C_2 \epsilon_j n] & \textrm{for } \frac{C_1}{ C_2 \epsilon_j} &< \sqrt{n} & \textrm{Phase III}
\end{align*}
We now give a regret bound for learner $i$ based on which phase its regret bound is in.
\paragraph{Regret bound of $\boldsymbol i$ in Phase I:} 
We first consider the case where $\bestallwell$ is in Phase I.
Then the balancing condition from \pref{lem:balancing} $\regretbound_i(n_i(t)) \leq 2\regretbound_\bestallwell(n_\bestallwell(t))$ implies that $n_i(t) / n_\bestallwell(t) \leq 2$ and thus
\begin{align*}
\regret_i(t) \leq 1 + 3\regretbound_\bestallwell(n_\bestallwell(t))  + 2(1 + \sqrt{2}) c \sqrt{n_i(t) \ln \frac{M \ln t}{\delta}}.
\end{align*}
If $\bestallwell$ is in Phase II, then by the balancing condition $n_i(t) \leq 4 C_1 \sqrt{n_\bestallwell(t)}$ which implies that $\frac{n_i(t)}{\sqrt{n_\bestallwell(t)}} \leq 4 C_1$. Plugging this into \eqref{eqn:basereg_before_case} yields
\begin{align*}
\regret_i(t) &\leq 1 + \regretbound_\bestallwell(n_\bestallwell(t)) + 2 c \sqrt{n_i(t) \ln \frac{M \ln t}{\delta}} 
+ \frac{n_i(t)}{\sqrt{n_\bestallwell(t)}} 2 C_1 + 8 c C_1\sqrt{ \ln \frac{\ln t}{\delta}}\\
&\leq 1 + \regretbound_\bestallwell(n_\bestallwell(t)) + 2 c \sqrt{n_i(t) \ln \frac{M \ln t}{\delta}} 
+ 8 C_1^2 + 8 c C_1\sqrt{ \ln \frac{M \ln t}{\delta}}.
\end{align*}
If $\bestallwell$ is in Phase III, then by the balancing condition $n_i(t) \leq 4 C_2 \epsilon_\bestallwell n_\bestallwell(t)$ and, hence, $\frac{n_i(t)}{n_\bestallwell(t)} \leq 4 C_2 \epsilon_\bestallwell \leq 4$. Here, we have used that $C_2 \epsilon_\bestallwell \leq 1$ as otherwise the regret bounds hold trivially.
Plugging this into \eqref{eqn:basereg_before_case} yields
\begin{align*}
\regret_i(t) \leq 1 + 5\regretbound_\bestallwell(n_\bestallwell(t)) + 6 c \sqrt{n_i(t) \ln \frac{M \ln t}{\delta}}.
\end{align*}

\paragraph{Regret bound of $\boldsymbol i$ in Phase II:} 
We here distinguish between two cases. If $\sqrt{n_\bestallwell(t)} \leq \frac{C_1}{C_2 \epsilon_\bestallwell}$, then $\regretbound_\bestallwell(n_\bestallwell(t)) \leq 2 C_1 \sqrt{n_\bestallwell(t)}$. Then by the balancing condition $\frac{n_i(t)}{n_\bestallwell(t)} \leq 9$. Plugging this into \eqref{eqn:basereg_before_case} yields
\begin{align*}
\regret_i(t) \leq 1 + 10\regretbound_\bestallwell(n_\bestallwell(t)) + 8 c \sqrt{n_i(t) \ln \frac{M \ln t}{\delta}}.
\end{align*}
Consider now the case where $\sqrt{n_\bestallwell(t)} > \frac{C_1}{C_2 \epsilon_\bestallwell}$ and $\regretbound_\bestallwell(n_\bestallwell(t)) \leq 2 \epsilon_\bestallwell C_2 n_\bestallwell(t)$.
Here, we bound \eqref{eqn:basereg_before_case} directly as
\begin{align*}
\regret_i(t) & \leq 1 + 2 \epsilon_\bestallwell C_2 (n_\bestallwell(t) + n_i(t)) + 2 c \sqrt{n_i(t) \ln \frac{M \ln t}{\delta}} 
 + 2 c \frac{C_2 \epsilon_\bestallwell}{C_1} n_i(t) \sqrt{\ln \frac{\ln t}{\delta}}\\
 &\leq 1 + 2 \epsilon_\bestallwell C_2 \left(n_\bestallwell(t) + n_i(t) + \frac{c\sqrt{\ln \frac{M \ln t}{\delta}}}{C_1}n_i(t)\right) + 2 c \sqrt{n_i(t) \ln \frac{M \ln t}{\delta}}.
\end{align*}

\paragraph{Regret bound of $\boldsymbol i$ in Phase III:} 
First, consider the case where $\sqrt{n_\bestallwell(t)} > \frac{C_1}{C_2 \epsilon_\bestallwell}$. Then we can directly write $\frac{n_i(t)}{n_\bestallwell(t)} \regretbound_\bestallwell(n_\bestallwell(t)) = \epsilon_\bestallwell C_2 n_i(t)$ and bound $1 / \sqrt{n_\bestallwell(t)} \leq \frac{C_2 \epsilon_\bestallwell}{C_1}$. Plugging this into \eqref{eqn:basereg_before_case} yields
\begin{align*}
\regret_i(t) \leq 1 + \regretbound_\bestallwell(n_\bestallwell(t)) + \epsilon_\bestallwell C_2 n_i(t) + 2 c \sqrt{n_i(t) \ln \frac{M \ln t}{\delta}} + \frac{C_2 \epsilon_\bestallwell}{C_1}2 c \sqrt{\ln \frac{M \ln t}{\delta}}n_i(t).
\end{align*}
It remains to bound the regret when $\sqrt{n_\bestallwell(t)} \leq \frac{C_1}{C_2 \epsilon_\bestallwell}$. Since $i$ is in Phase III, we also have $\sqrt{n_i(t)} > \frac{C_1}{C_2 \epsilon_i} \geq \frac{C_1}{C_2 \epsilon_\bestallwell}$. The balancing condition yields $\epsilon_i C_2 n_i(t) \leq 4 C_1 \sqrt{n_\bestallwell(t)}$ and thus 
\begin{align*}
    \frac{n_i(t)}{\sqrt{n_\bestallwell(t)}} \leq \frac{4 C_1}{C_2 \epsilon_i} \leq \sqrt{n_i(t)}.
\end{align*}
Plugging this into \eqref{eqn:basereg_before_case} yields
\begin{align*}
\regret_i(t) &\leq 1 + \regretbound_\bestallwell(n_\bestallwell(t)) + 4 c \sqrt{n_i(t) \ln \frac{M \ln t}{\delta}} 
+ \frac{n_i(t)}{n_\bestallwell(t)} 2 C_1 \sqrt{n_\bestallwell(t)}\\
&\leq 1 + \regretbound_\bestallwell(n_\bestallwell(t)) + 4 c \sqrt{n_i(t) \ln \frac{M \ln t}{\delta}} 
+ 2 C_1 \sqrt{n_i(t)}.
\end{align*}
\end{proof}

%% file: appendix/adversarial_proofs.tex
\section{Proofs for Setting with Adversarial Contexts}

\subsection{Epoch Balancing Termination (Proof of \pref{thm:adv_epoch_balancing_termination})}
\advepochnonterminate*
\begin{proof}
Since all learners are well-specified and their lower-confidence bounds $L_{t,i}$ satisfy $L_{t, i} \leq \EE[r_t |  a_{t, i}, x_t] \leq \mu^\star_k$, the right-hand side of the misspecification test satisfies
\begin{align*}
    \max_{j \in \Ical} \sum_{k = 1}^t B_{k, j} \leq \sum_{k=1}^t \mu^\star_k.
\end{align*}
for all $t \in \NN$
Further, with probability at least $1 - \delta$, by \pref{lem:test_LHS}, the left-hand side of the misspecification test satisfies for all $t \in \NN$
\begin{align*}
\sum_{i \in \Ical} [U_i(t) + \regretbound_i(n_i(t))] + c \sqrt{t \ln\frac{\ln(t)} {\delta}}
\geq
\sum_{k=1}^t \mu^\star_k.
\end{align*}
Thus, the misspecification test never triggers and \pref{alg:adversarial_epoch_balancing} does not terminate.
\end{proof}

\begin{lemma}\label{lem:mu_adv_conc}
Let $\delta \in (0,1)$ and consider the event
\begin{align*}
  \Gcal = \left\{ \forall t \in \NN \colon   \left| \sum_{i \in \Ical} U_i(t) - \sum_{k = 1}^t \EE[r_k | a_k, x_k]\right|  \leq  c \sqrt{t \ln\frac{\ln(t)} {\delta}} \right\}.    
\end{align*}
where $c > 0$ is an absolute constant. Then $\PP(\Gcal) \geq 1 - \delta$.
\end{lemma}
\begin{proof}
Let $\Fcal_t = \sigma(x_1, i_1, a_1, r_1, \dots, x_{t-1}, i_{t-1}, a_{t-1}, r_{t-1}, x_{t-1}, i_{t-1}, a_{t-1})$ be the sigma-field induced by all variables up to the reward at round $t$. Hence, $X_k = r_k - \EE[r_k | a_k, x_k]$ is a martingale-difference sequence w.r.t. $\Fcal_k$. We will now apply a Hoeffding-style uniform concentration bound from \citet{howard2018uniform}.
Using the terminology and definition in this article, by case Hoeffding~I in Table~4, the process  $S_k = \sum_{j=1}^k X_k$ is sub-$\psi_N$ with variance process $V_k = k / 4$.
Thus by using the boundary choice in Equation~(11) of \citet{howard2018uniform}, we get
\begin{align*}
    S_k & \leq 1.7 \sqrt{V_k \left( \ln \ln(8 V_k) + 0.72 \ln(5.2 / \delta) \right) }\\
    & = 0.85\sqrt{k\left( \ln \ln(4k) + 0.72 \ln(5.2 / \delta)\right)}
\end{align*}
for all $k$ with probability at least $1 - \delta$.
Applying the same argument to $-S_k$ gives that
\begin{align*}
    \left| \sum_{k=1}^t \left(r_k - \EE[ r_k | a_k, x_k]\right)\right|
    \leq 0.85\sqrt{t\left( \ln \ln(4t) + 0.72 \ln(10.4 / \delta)\right)}
\end{align*}
holds with probability at least $1 - \delta$ for all $t$. Since $\sum_{i \in \Ical} U_i(t) = \sum_{k=1}^t r_k$, the statement follows. Note that this concentration argument holds for all $t$ uniformly and therefore also when $t$ is random.
\end{proof}

\begin{lemma}[Upper-confidence bound on optimal reward]
\label{lem:test_LHS}
In event $\Gcal$ from \pref{lem:mu_adv_conc}, the following holds.
If at time $t$ all learners $i \in \Ical$ are well-specified, then the left-hand side in the misspecification test of \pref{alg:adversarial_epoch_balancing} is a lower-bound on the optimal rewards, i.e., 
\begin{align*}
\sum_{i \in \Ical} [U_i(t) + \regretbound_i(n_i(t))] + c \sqrt{t \ln\frac{\ln(t)} {\delta}}
\geq
\sum_{k=1}^t \mu^\star_k.
\end{align*}
\end{lemma}
\begin{proof}
By \pref{lem:mu_adv_conc}, in the considered event, we have
\begin{align*}
&\sum_{i \in \Ical} [U_i(t) + \regretbound_i(n_i(t))] + c \sqrt{t \ln\frac{\ln(t)} {\delta}}\\
&\geq \sum_{i \in \Ical} \regretbound_i(n_i(t)) + \sum_{k = 1}^t \EE[r_k | a_k, x_k] &\textrm{(by \pref{lem:mu_adv_conc})}\\
&\geq \sum_{i \in \Ical} \regret_i(t) + \sum_{k = 1}^t \EE[r_k | a_k, x_k] &\textrm{(each learner is well-specified)}\\
&=\sum_{i \in \Ical} \left[ \regret_i(t) + \sum_{k \in T_i(t)} \EE[r_k | a_k, x_k] \right]
\\
&=\sum_{i \in \Ical}  \sum_{k \in T_i(t)} \mu^\star_k = \sum_{k=1}^t \mu^\star_k.& \textrm{(by definition of regret)}
\end{align*}
\end{proof}

\subsection{Regret Bound for Epoch Balancing (Proof of \pref{thm:adv_epoch_balancing_regret})}
\advepochregret*
\begin{proof}
We apply \pref{thm:adv_general_regret} which immediately yields the desired bound
\begin{align*}
    \regret(t) \leq \tilde O 
    \left( 
    \left(d_\star + \sqrt{d_\star}S_\star  + |\Ical| \right)
    \textcolor{DarkGreen}{(d_\star + \sqrt{d_\star}S_\star)\sqrt{t}}
    \right)\ .
\end{align*}
\end{proof}

\advepochregretls*
\begin{proof}
We apply \pref{thm:adv_general_regret} which yields
\begin{align*}
    \regret(t) \leq \tilde O 
    \left( 
    \left(d_\star L_\star  + \sqrt{d_\star}S_\star L_\star  + |\Ical| \right)
    \textcolor{DarkGreen}{(d_\star + \sqrt{d_\star}S_\star)L_\star \sqrt{t}} + \sum_{i \in \Ical} L_i S_i \ln \ln (t)
    \right)\ .
\end{align*}
\end{proof}

\begin{theorem}[General Regret Bound of Epoch Balancing]
\label{thm:adv_general_regret}
Assume that \pref{alg:adversarial_epoch_balancing} is run with instances of OFUL as base learners which use different dimensions $d_i, S_i, L_i, R^{\max}_i$ and regularization parameter $\lambda = 1$. Denote by $\star$ the index of the base learner so that all base learners $j \in \Ical$ with $d_j \geq d_\star$ are well-specified and their elliptical confidence sets always contain the true parameter. Then, with probability at least $1 - 2 \delta$, the regret is bounded for all rounds $t$ as
\begin{align*}
    \regret(t) & \leq (|\Ical| \sqrt{z_\star} + z_\star \sqrt{\bar M}) x(t) \sqrt{t}  + 8.12 \sum_{i \in \Ical} R^{\max}_i \ln \frac{5.2|\Ical|\ln \left( 2 t\right)}{\delta}+ 
 2 c \sqrt{t \ln\frac{\ln(t)} {\delta}}
 \\
 &\leq 
 \sqrt{(d_\star^2 + d_\star S_\star^2)}|\Ical|\left( R^{\max}_\star \wedge L_\star\right)\sqrt{t}(2 + 2c )x(t)\\
 &\qquad + (d_\star^2 + d_\star S_\star^2) \left( R^{\max}_\star \wedge L_\star\right)^2 \sqrt{\bar M t}
 (2 + 2c )x(t)\\
 &\qquad +
 8.12 \sum_{i \in \Ical} R^{\max}_i \ln \frac{5.2|\Ical|\ln \left( 2 t\right)}{\delta},
\end{align*}
where $\bar M = |\Ical|$ for general $z_i$ and $\bar M = 2$ when $z_i$ are exponentially increasing (i.e., $2z_i \leq z_{i+1}$ for all $i \in \Ical$). Here $x(t) = O(\ln\frac{tL_{\max} }{\delta} + \ln \ln (R_{\max}^{\max}t \wedge L_{\max}t)$
\end{theorem}
\begin{proof}
Since learner $i_\star$ is well-specified and its elliptical confidence set contains $\theta^\star$, it holds that 
\begin{align*}
    \sum_{k=1}^t \mu^\star_k \leq \sum_{k=1}^t \max_{a \in \Acal_k} \left[ \langle  \wh \theta_{k, \star}, a \rangle + \beta_{k, \star} \|a\|_{\Sigma_{k, \star}^{-1}} \right]
    = \sum_{k=1}^t  \langle  \wh \theta_{k, \star}, a_{k, \star} \rangle + \beta_{k, \star} \|a_{k, \star}\|_{\Sigma_{k, \star}^{-1}}.
\end{align*}
Thus, we can write the total regret up to round $t$ as
\begin{align*}
    \regret(t) &=
    \sum_{k=1}^t \left[ \mu^\star_k - \EE[ r_k | a_k, x_k] \right]
    = \sum_{k=1}^t \mu^\star_k - \sum_{k=1}^t \EE[ r_k | a_k, x_k] \\
    & \leq \sum_{k=1}^t \mu^\star_k - \sum_{i \in \Ical} U_i(n_i(t)) + c \sqrt{t \ln\frac{\ln(t)} {\delta}},
\end{align*}
where the inequality holds in event $\Gcal$ of \pref{lem:mu_adv_conc}. If \pref{alg:adversarial_epoch_balancing} does not stop in iteration $t$, then the misspecification test does not trigger for any learner, and in particular for learner $i_\star$. This implies that
\begin{align*}
    \sum_{i \in \Ical} [U_i(t) + \regretbound_i(n_i(t))] + c \sqrt{t \ln\frac{\ln(t)} {\delta}}
    \geq 
    \sum_{k=1}^t B_{k, \star} %
\end{align*}
Rearranging terms and plugging this inequality back into the regret bound from above yields
\begin{align}
\regret(t) &
\leq \sum_{k=1}^t \left[\mu^\star_k - B_{k, \star}\right] + \sum_{i \in \Ical} R_i(n_i(t)) + 2 c \sqrt{t \ln\frac{\ln(t)} {\delta}}
     \label{eqn:regretbound_adv_epoch1}
\end{align}
We bound the first term in \pref{eqn:regretbound_adv_epoch1} as 
\begin{align*}
    &\sum_{k=1}^t \left[\mu^\star_k - B_{k, \star}\right]
    \\&\overset{(i)}{\leq}  \sum_{k=1}^t \left[R^{\max}_\star \wedge (\langle  \wh \theta_{k, \star}, a_{k, \star} \rangle +  \beta_{k, \star} \|a_{k, \star}\|_{\Sigma_{k, \star}^{-1}}) - (-R^{\max}_\star \vee (\langle  \wh \theta_{k, \star}, a_{k, \star} \rangle -  \beta_{k, \star} \|a_{k, \star}\|_{\Sigma_{k, \star}^{-1}}))\right]\\
    &\leq \sum_{k=1}^t \left[2R^{\max}_\star \wedge 2\beta_{k, \star} \|a_{k, \star}\|_{\Sigma_{k, \star}^{-1}}\right]
    \leq 2 \beta_{t, \star} \sum_{k=1}^t \left[\frac{R^{\max}_\star}{\beta_{t, \star}} \wedge \|a_{k, \star}\|_{\Sigma_{k, \star}^{-1}}\right]\\
    &\overset{(ii)}{\leq} 2 \beta_{t, \star} \sqrt{t\sum_{k=1}^t \left[\left(\frac{R^{\max}_\star}{\beta_{t, \star}}\right)^2 \wedge \frac{L^2}{\lambda_i}\wedge \|a_{k, \star}\|^2_{\Sigma_{k, \star}^{-1}}\right]}
\end{align*}
where $(i)$ follows from the definition of $B_{k, i}$ and the fact that the ellipsoid confidence set of $\star$ contain the true parameter and $(ii)$ applies the Cauchy-Schwarz inequality. We now apply a randomized version of the elliptical potential lemma which we prove in \pref{lem:elliptical_random}. This yields
\begin{align*}
    \sum_{k=1}^t \left[\mu^\star_k - B_{\star, k}\right]
    &\leq
    4 \beta_{t, \star} \sqrt{
    \frac{t}{p_\star}(1 + b_\star^2) \ln \frac{5.2\ln(2b_\star^2 t \vee 2) \det \Sigma_{t, \star}}{\delta \det \Sigma_{0, \star}}
    }\\
    &\leq
    4 \beta_{t, \star} \sqrt{
    \frac{t d_\star}{p_\star}(1 + b_\star^2) \ln \frac{5.2\ln(2b_\star^2 t \vee 2) (d_\star \lambda_\star + t L_\star^2)}{\delta d_\star \lambda_\star}
    }
\end{align*}
where $b_\star = \frac{R^{\max}_\star}{\beta_{t, \star}} \wedge \frac{L_\star}{\sqrt{\lambda_\star}}$.
For the second term in \pref{eqn:regretbound_adv_epoch1}, we apply
 \pref{lem:regret_bounds_balanced_prob} with $\alpha = \delta$ as
 \begin{align*}
\sum_{i \in \Ical} R_i(n_i(t))
& \leq 
    8.12 \sum_{i \in \Ical} R^{\max}_i \ln \frac{5.2|\Ical|\ln \left( 2 t\right)}{\delta} %
 + 2 \sum_{i \in \Ical} \beta_{t, i} \sqrt{3 d_i p_i t \left( 1 + b_i^2\right) \ln \frac{d_i \lambda_i +  t p_i L^2_i }{d_i \lambda_i}}.
\end{align*}
Combining the terms for both bounds, we arrive at the regret bound
\begin{align*}
    \regret(t) &
\leq 
4 \beta_{t, \star} \sqrt{
    \frac{t d_\star}{p_\star}(1 + b_\star^2) \ln \frac{5.2\ln(2b_\star^2 t \vee 2) (d_\star \lambda_\star + t L_\star^2)}{\delta d_\star \lambda_\star}
    }\\
    & \qquad + 2 \sum_{i \in \Ical} \beta_{t, i} \sqrt{3 d_i p_i t \left( 1 + b_i^2\right) \ln \frac{d_i \lambda_i +  t p_i L^2_i }{d_i \lambda_i}}
    \\
    & \qquad + 
    8.12 \sum_{i \in \Ical} R^{\max}_i \ln \frac{5.2|\Ical|\ln \left( 2 t\right)}{\delta}
    + 
 2 c \sqrt{t \ln\frac{\ln(t)} {\delta}}\\
 & \leq x(t) \sqrt{\frac{z_\star t}{p_\star}} + x \sum_{i \in \Ical}\sqrt{z_i p_i t} 
    + 
    8.12 \sum_{i \in \Ical} R^{\max}_i \ln \frac{5.2|\Ical|\ln \left( 2 t\right)}{\delta}+ 
 2 c \sqrt{t \ln\frac{\ln(t)} {\delta}}
\end{align*}
where
\begin{align*}
z_i &= ( \sigma^2 d_i + \lambda_i S_i^2) d_i (1 + b_i^2) \leq 2 (d_i^2 + d_i S_i^2) \left( R^{\max}_i \wedge L_i\right)^2
& \textrm{and}\\
x(t) &= 12 \max_{i \in \Ical}\sqrt{\ln \left(\frac{1 + t L_i^2 / \lambda_i}{\delta} \right) \ln \frac{5.2\ln(2b_i^2 t \vee 2) (d_i \lambda_i + t L_i^2)}{\delta d_i \lambda_i}}\\
&\leq 12 \max_{i \in \Ical}\sqrt{\ln \left(\frac{1 + t L_i^2}{\delta} \right) \ln \frac{10.4\ln(2\left( R^{\max}_i \wedge L_i\right) t) (1 + t L_i^2)}{\delta}}\\
&\leq
12 \ln \frac{10.4 (1 + t L_{\max}^2)\ln(2\left( R^{\max}_{\max} \wedge L_{\max}\right) t)}{\delta}.
\end{align*}
We now use the definition of $p_i \propto \frac{1}{z_i}$ 
and bound
\begin{align*}
     \sum_{i \in \Ical}\sqrt{z_i p_i} 
     = \sum_{i \in \Ical} \sqrt{\frac{1}{\sum_{i \in \Ical} z_i^{-1}}}
     = \frac{|\Ical|}{\sqrt{\sum_{i \in \Ical} z_i^{-1}}}
     \leq \frac{|\Ical|}{\sqrt{z_\star^{-1}}} = |\Ical| \sqrt{z_\star}
\end{align*}
where the inequality uses the fact that $\star \in \Ical$. Further
\begin{align*}
    \sqrt{\frac{z_\star}{p_\star}} 
    = z_\star \sqrt{\sum_{i \in \Ical} \frac{1}{z_i} } \leq z_\star \sqrt{|\Ical|}
\end{align*}
holds for any $z_i$ but if we know that $z_1 \leq 2 z_2 \leq 4 z_4 \dots M z_M$, then 
\begin{align*}
    \sqrt{\frac{z_\star}{p_\star}} 
    = z_\star \sqrt{\sum_{i \in \Ical} \frac{1}{z_i} } \leq 2z_\star.
\end{align*}
Thus, we can bound the total regret as
\begin{align*}
    \regret(t) & \leq (|\Ical| \sqrt{z_\star} + z_\star \sqrt{\bar M}) x(t) \sqrt{t}  + 8.12 \sum_{i \in \Ical} R^{\max}_i \ln \frac{5.2|\Ical|\ln \left( 2 t\right)}{\delta}+ 
 2 c \sqrt{t \ln\frac{\ln(t)} {\delta}}
 \\
 &\leq 
 \sqrt{(d_\star^2 + d_\star S_\star^2)}|\Ical|\left( R^{\max}_\star \wedge L_\star\right)\sqrt{t}(2 + 2c )x(t)\\
 &\qquad + (d_\star^2 + d_\star S_\star^2) \left( R^{\max}_\star \wedge L_\star\right)^2 \sqrt{\bar M t}
 (2 + 2c )x(t)\\
 &\qquad +
 8.12 \sum_{i \in \Ical} R^{\max}_i \ln \frac{5.2|\Ical|\ln \left( 2 t\right)}{\delta},
\end{align*}
where $\bar M = |\Ical|$ for general $z_i$ and $\bar M = 2$ when $z_i$ are exponentially increasing.
Note that since this bound holds in the penultimate round of \pref{alg:adversarial_epoch_balancing} and the regret in the final round can be at most $1$, this bound holds for all rounds $t$  played by \pref{alg:adversarial_epoch_balancing}, including the last.
\end{proof}

\begin{lemma}[Regret bounds are balanced]
\label{lem:regret_bounds_balanced_prob}
Let $\alpha \in (0,1)$ be arbitrary but fixed. With probability at least $1 - \alpha$, the sum of regret bounds satisfy in all iterations $t$ of \pref{alg:adversarial_epoch_balancing} the following upper-bound 
\begin{align*}
\sum_{i \in \Ical} R_i(n_i(t))
& \leq 
    8.12 \sum_{i \in \Ical} R^{\max}_i \ln \frac{5.2|\Ical|\ln \left( 2 t\right)}{\alpha} %
 + 2 \sum_{i \in \Ical} \beta_{t, i} \sqrt{3 d_i p_i t \left( 1 + b_i^2\right) \ln \frac{\lambda_i d_i +  3tp_i L^2_i}{\lambda_i d_i}}
\end{align*}
where $b_i = \frac{R^{\max}_i}{2 \beta_{t, i}} \wedge \frac{L_i}{\sqrt{\lambda_i}}$.
\end{lemma}
\begin{proof}
By the choice of regret bounds we have
\begin{align*}
    R_i(n_i(t)) & =  \sum_{k \in T_i(t)}\left[ 2\beta_{ k, i} \|a_{k, i}\|_{\Sigma_{k, i}^{-1}}
     \wedge R^{\max}_i\right]
    \\
    &\leq  
    R^{\max}_i n_i(t) \wedge 2  \beta_{ t, i}
    \sum_{k \in T_i(t)}\left(  \|a_{k, i}\|_{\Sigma_{k, i}^{-1}} \wedge \frac{R^{\max}_i}{2 \beta_{ t, i}} \right)
    \\
    &\leq
        R^{\max}_i n_i(t) \wedge 2  \beta_{ t, i}
    \sqrt{n_i(t) \sum_{k \in T_i(t)}\left(  \|a_{k, i}\|^2_{\Sigma_{k, i}^{-1}} \wedge \left(\frac{R^{\max}_i}{2 \beta_{t, i}}\right)^2 \wedge \frac{L_i^2}{\lambda_i}\right)}
    \\
   & \leq R^{\max}_i n_i(t) \vee 2 \beta_{t, i} \sqrt{d_i n_i(t) \left( 1 + b_i^2\right) \ln \frac{\lambda_i +  n_i(t) L_i^2 / d_i}{\lambda_i}} 
\end{align*} 
where $b_i = \frac{R^{\max}_i}{2 \beta_{t, i}} \wedge \frac{L_i}{\sqrt{\lambda_i}}$ and the last inequality follows from of \pref{lem:elliptical}. 
To control the the number of times each learner was chosen, we use \pref{lem:playcount_ub}. This gives with probability at least $1 - \alpha$ for all iterations $t$ simultaneously
$ n_i(t) \leq 3 t p_i \vee 8.12 \ln \frac{5.2|\Ical|\ln \left( 2 t\right)}{\alpha}$.
This yields a regret bound of
\begin{align*}
    R_i(n_i(t)) & \leq 
    8.12 R^{\max}_i \ln \frac{5.2|\Ical|\ln \left( 2 t\right)}{\alpha} 
    \quad \vee  \quad 2 \beta_{t, i} \sqrt{3 d_i p_i t \left( 1 + b_i^2\right) \ln \frac{\lambda_i +  3 t p_i L^2_i / d_i}{\lambda_i}}.
\end{align*} 
Summing over $R_i$ and plugging in $\beta_{t, i}$ yields
\begin{align*}
\sum_{i \in \Ical} R_i(n_i(t))
& \leq 
    8.12 \sum_{i \in \Ical} R^{\max}_i \ln \frac{5.2|\Ical|\ln \left( 2 t\right)}{\alpha} %
 + 2 \sum_{i \in \Ical} \beta_{t, i} \sqrt{3 d_i p_i t \left( 1 + b_i^2\right) \ln \frac{\lambda_i +  3 t p_i L^2_i / d_i}{\lambda_i}}
\end{align*}

\end{proof}

\begin{lemma}\label{lem:playcount_ub}
The number of times each a learner $i \in \Ical$ has been played in \pref{alg:adversarial_epoch_balancing} after $t$ iterations is bounded with probability at least $1 - \delta$ for all $t \in \NN$ and $i \in \Ical$ as
\begin{align*}
    n_i(t) \leq \frac{3}{2}tp_i +4.06 \ln \frac{5.2|\Ical|\ln \left( 2 t\right)}{\delta}
    \leq 3 t p_i \vee 8.12 \ln \frac{5.2|\Ical|\ln \left( 2 t\right)}{\delta}
\end{align*}
\end{lemma}
\begin{proof}
Fix an $i \in \Ical$ and consider the martingale difference sequence $X_t = \one\{i_t = i\} - p_i$ with variance. The process $S_t = \sum_{k=1}^t X_k$ with variance process $W_t = t p_i(1-p_i)$ satisfies the sub-$\psi_P$ condition of \citet{howard2018uniform} with constant $c = 1$ (see Bennett case in Table~3 of \citet{howard2018uniform}). By \pref{lem:uniform_emp_bernstein}, the bound
\begin{align*}
   S_t \leq &~ 1.44 \sqrt{(W_t \vee m) \left( 1.4 \ln \ln \left(2 (W_t/m \vee 1)\right) + \ln \frac{5.2}{\delta}\right)}\\
    &+ 0.41 \frac{L^2}{\lambda}  \left( 1.4 \ln \ln \left( 2 (W_t/m \vee 1)\right) + \ln \frac{5.2}{\delta}\right)
\end{align*}
holds for all $t \in \NN$ with probability at least $1 - \delta$. We set $m = t p_i$ and upper-bound the RHS further as
\begin{align*}
   S_t &\leq 1.44 \sqrt{t p_i \left( 1.4 \ln \ln \left(2 t\right) + \ln \frac{5.2}{\delta}\right)}
    + 0.41 \left( 1.4 \ln \ln \left( 2 t\right) + \ln \frac{5.2}{\delta}\right)
    \\ &\leq \frac{t p_i}{2} +  1.45 \left( 1.4 \ln \ln \left( 2 t\right) + \ln \frac{5.2}{\delta}\right),
\end{align*}
where used the AM-GM inequality in the final step.
We therefore get that with probability at least $1 - \delta$, the following upper-bound in the number of times learner $i$ was selected by time $t$ holds for all $i \in \Ical$ and $t \in \NN$:
\begin{align*}
    n_i(t) \leq\frac{3}{2}tp_i + 2.9 \left( 1.4 \ln \ln \left( 2 t\right) + \ln \frac{5.2|\Ical|}{\delta}\right)  \leq \frac{3}{2}tp_i +4.06 \ln \frac{5.2|\Ical|\ln \left( 2 t\right)}{\delta}.
\end{align*}
We can now distinguish between two cases: 
When $\frac{3}{2} t p_i  \leq 4.06 \ln \frac{5.2|\Ical|\ln \left( 2 t\right)}{\delta}$, then
\begin{align*}
    n_i(t) \leq 8.12 \ln \frac{5.2|\Ical|\ln \left( 2 t\right)}{\delta}
\end{align*}
and otherwise
$n_i(t) \leq 3 t p_i$.
\end{proof}

%% file: appendix/misc_appendix.tex
\section{Ancillary Technical Lemmas}\label{app:ancillary}
\begin{lemma}[Regret Bound for \oful] \label{lem:LinUCB_regret_guarantee}
Assume \oful (\pref{alg:OFUL}) uses regularization parameter $\lambda > 0$ 
chooses the each action as 
\begin{align*}
    a_t \in \argmax_{a \in \Acal_t} \langle \wh \theta_t, a\rangle + \beta_t \|a\|_{V_t^{-1}},
\end{align*} 
where $\theta_t$ is a parameter estimate, $\beta_t \in \RR$ is a confidence width and $V_t \succcurlyeq \lambda I + \sum_{l=1}^{t-1} a_l a_l^\top$ is a covariance matrix. In the event that the true parameter $\theta_\star$ was contained at all times in the confidence ellipsoid, that is, $\|\theta_\star - \hat \theta_t\|_{V_{t}} \leq \beta_t$ for all $t \in [T]$, the (pseudo-)regret is bounded as
\begin{align*}
    \regret(T) 
    \leq 2 \beta_{\max}
    \sqrt{dT \left(1 + \frac{L^2}{\lambda}\right) \ln \frac{d \lambda + TL^2}{d \lambda}},
\end{align*}
where $\beta_{\max} = \max_{t \in [T]} \beta_t$ is the largest confidence width during all rounds and $L = \max_{a \in \bigcup_t\Acal_t}\|a\|_2$ be a bound on the action norms.
\end{lemma}
\begin{remark}
This regret bound for OFUL holds for any, possibly random, sequence of confidence widths as long as the true parameter is contained in the confidence ellipsoid. It does not assume any specific form or monotonicity or $\beta_t \geq 1$. It also does not prescribe that the covariance matrix exactly matches $\lambda I + \sum_{l=1}^{t-1} a_l a_l^\top$. This makes this regret bounds applicable to the case where $\hat \theta_t$ includes additional observations besides the ones from previous rounds played by the algorithm.
\end{remark}
\begin{proof}
The immediate regret at time $t$ (defined as the difference of the expected reward of the optimal action choice $a^\star_t \in \argmax_{a \Acal_t} \langle \theta_\star, a\rangle$ and the action $a_t$ taken by the algorithm) is bounded as
\begin{align*}
    \langle \theta_\star, a_t^\star - a_t\rangle
    & \overset{(i)}{\leq} \langle \wh \theta_t, a_t^\star \rangle
    + \beta_t\|a_t^\star \|_{V_{t}^{-1}}
    - \langle \theta_\star, a_t\rangle
    \\
    & \overset{(ii)}{\leq} \langle \wh \theta_t, a_t \rangle
    + \beta_t\|a_t \|_{V_{t}^{-1}}
    - \langle \theta_\star, a_t\rangle
        \\
    & \overset{(iii)}{\leq} 2\beta_t\|a_t \|_{V_{t}^{-1}} 
    \overset{(iv)}{\leq} 2\beta_t\|a_t \|_{\Sigma_{t}^{-1}},
\end{align*}
where $\Sigma_t = \lambda I + \sum_{l=1}^{t-1} a_l a_l^\top$. 
Step $(i)$ follows from $\|\theta_\star - \hat \theta_t\|_{V_{t}} \leq \beta_t$, step $(ii)$ from the algorithm's action choice and step $(iii)$ again from the confidence ellipsoid $\|\theta_\star - \hat \theta_t\|_{V_{t}} \leq \beta_t$. Finally, step $(iv)$ follows from the assumption that $V_t \succcurlyeq \lambda I + \sum_{l=1}^{t-1} a_l a_l^\top = \Sigma_t$.

Since $L$ is a bound of the action norm and $\Sigma_t \succcurlyeq \lambda I$, we have $\|a_t \|_{\Sigma_{t}^{-1}} = \|\Sigma_{t}^{-1/2} a_t\|_2 \leq  \frac{L}{\sqrt{\lambda}}$.
Thus, we can bound the regret as
\begin{align*}
    \regret(T) 
    & \leq 
    2\sum_{t=1}^T  \beta_t \|a_t \|_{\Sigma_{t}^{-1}}\\
        &\leq 2\sqrt{\sum_{t=1}^T \beta_t^2 } \sqrt{ \sum_{t=1}^T \|a_t \|_{\Sigma_{t}^{-1}}^2} && \textrm{(Cauchy-Schwarz)}\\
    &\leq 2 \beta_{\max} \sqrt{T \sum_{i=1}^T \frac{L^2}{\lambda} \wedge \|a_t \|_{\Sigma_{t}^{-1}}^2} \\
    &\leq 
    2 \beta_{\max} \sqrt{T \left(1 + \frac{L^2}{\lambda}\right) \ln \frac{\det \Sigma_{T+1}}{\det \Sigma_1}}
    && \textrm{(\pref{lem:elliptical} below)}
    \\
    & \leq 2
    \beta_{\max} \sqrt{dT \left(1 + \frac{L^2}{\lambda}\right) \ln \frac{d \lambda + TL^2}{d \lambda}}.
\end{align*}
\end{proof}

\begin{lemma}[Elliptical potential]
\label{lem:elliptical}
Let $x_1, \dots, x_n \in \mathbb R^d$ and $V_t = V_0 + \sum_{i=1}^t x_i x_i^\top$ and $b > 0$ then
\begin{align*}
    \sum_{t=1}^n b \wedge \|x_t \|_{V_{t-1}^{-1}}^2 \leq \frac{b}{\ln(b + 1)} \ln \frac{\det V_n}{\det V_0} \leq (1 + b) \ln \frac{\det V_n}{\det V_0}.
\end{align*}
\end{lemma}
\begin{proof}[Proof Sketch]
The proof is identical to the usual elliptical potential lemma \citep[Lemma~19.4]{lattimore2018bandit} where $b = 1$ except that we need to argue that for any $b > 0$
\begin{align*}
    b \wedge u \leq c\ln(u + 1)
\end{align*}
holds whenever $c \geq \frac{b}{\ln(1 + b)}$. Since $\ln(1 + \cdot)$ is strictly concave and strictly monotonically increasing, it is sufficient for us to check that this inequality holds at the critical point $u = b$ which is the case.
\end{proof}

\begin{lemma}[Randomized elliptical potential]
\label{lem:elliptical_random}
Let $x_1, x_2, \dots \in \mathbb R^d$ and $I_1, I_2, \dots \in \{0,1\}$ and $V_0 \in \RR^{d \times d}$ be random variables so that
$\EE[I_k | x_1, I_1, \dots, x_{k-1}, I_{k-1}, x_k, V_0] = p$ for all $k \in \NN$. Further, let  $V_t = V_0 + \sum_{i=1}^t I_i x_i x_i^\top$. Then
\begin{align*}
    \sum_{t=1}^n b \wedge \|x_t \|_{V_{t-1}^{-1}}^2 
    &\leq 1 \vee
    2.9 \frac{b}{p}  \left( 1.4 \ln \ln \left( 2bn \vee 2\right) + \ln \frac{5.2}{\delta}\right) + \frac{2}{p} \left(1 + b\right) \ln \frac{\det V_n}{\det V_0}
    \\
    & =  \frac{4}{p}(1 + b) \ln \frac{\ln(2bn \vee 2) 5.2\det V_n}{\delta \det V_0}
\end{align*}
holds with probability at least $1 - \delta$ for all $n$ simultaneously.
\end{lemma}
\begin{proof}
We decompose the sum of squares as
\begin{align}
\label{eqn:sumsq1}
    \sum_{t=1}^n b \wedge \|x_t \|_{V_{t-1}^{-1}}^2
    = \frac{1}{p}\sum_{t=1}^n  (b I_t \wedge \|I_t x_t \|_{V_{t-1}^{-1}}^2) + 
    \frac{1}{p}\sum_{t=1}^n (p - I_t) (b \wedge \|x_t \|_{V_{t-1}^{-1}}^2 ) 
\end{align}
The first term can be controlled using the standard elliptical potential lemma in \pref{lem:elliptical} as
\begin{align*}
   \frac{1}{p}\sum_{t=1}^n  (b I_t \wedge \|I_t x_t \|_{V_{t-1}^{-1}}^2)
   \leq \frac{1}{p}\sum_{t=1}^n  (b \wedge \|I_t x_t \|_{V_{t-1}^{-1}}^2)
    \leq \frac{1}{p} \left(1 + b\right) \ln \frac{\det V_n}{\det V_0}.
\end{align*}
For the second term, we apply an empirical variance uniform concentration bound. 
Let $\Fcal_{i-1} = \sigma(V_0, x_1, I_1, \dots, x_{i-1}, I_{i-1}, x_i)$ be the sigma-field up to before the $i$-th indicator. 
Let $Y_i = \frac{1}{p} (p - I_i) \left(\|x_i \|^2_{V_{i-1}^{-1}} \wedge b\right)$ which is a martingale difference sequence because $\EE[Y_i | \Fcal_{i-1}] = 0$ and consider the process $S_t =  \sum_{i=1}^t Y_i$ with variance process
\begin{align*}
W_t &= \sum_{i=1}^t \EE[Y_i^2 | \Fcal_{i-1}] 
= \sum_{i=1}^t  \frac{1}{p^2} \left(\|x_i \|^2_{V_{i-1}^{-1}} \wedge b\right)^2\EE[(p - I_i)^2 | \Fcal_{i-1}]\\
&= \frac{1-p}{p} \sum_{i=1}^t \left(\|x_i \|^2_{V_{i-1}^{-1}} \wedge b\right)^2 \leq \frac{b}{p } \sum_{i=1}^t \left(\|x_i \|^2_{V_{i-1}^{-1}} \wedge b\right) \leq \frac{tb^2}{p}.
\end{align*}
Note that $Y_t \leq b$ and therefore, $S_t$ satisfies with variance process $W_t$ the sub-$\psi_P$ condition of  \citet{howard2018uniform} with constant $c = b$ (see Bennett case in Table~3 of \citet{howard2018uniform}). By \pref{lem:uniform_emp_bernstein} below, the bound
\begin{align*}
   S_t \leq &~ 1.44 \sqrt{(W_t \vee m) \left( 1.4 \ln \ln \left(2 (W_t/m \vee 1)\right) + \ln \frac{5.2}{\delta}\right)}\\
    &+ 0.41 b  \left( 1.4 \ln \ln \left( 2 (W_t/m \vee 1)\right) + \ln \frac{5.2}{\delta}\right)
\end{align*}
holds for all $t \in \NN$ with probability at least $1 - \delta$. We set $m = \frac{b}{p}$ and upper-bound the RHS further as
\begin{align*}
   & 1.44 \sqrt{\frac{b}{p}\left(1 \vee \sum_{i=1}^t \left(b \wedge \|x_i \|^2_{V_{i-1}^{-1}}\right) \right)  \left( 1.4 \ln \ln \left(2bt \vee 2\right) + \ln \frac{5.2}{\delta}\right)}\\
&    + 0.41 b  \left( 1.4 \ln \ln \left(2bt \vee 2\right) + \ln \frac{5.2}{\delta}\right)\\
&\leq   \frac{1}{2}\left(1 \vee \sum_{i=1}^t \left(b \wedge \|x_i \|^2_{V_{i-1}^{-1}}\right) \right)  + 1.45 \frac{b}{p}  \left( 1.4 \ln \ln \left(2bt \vee 2\right) + \ln \frac{5.2}{\delta}\right),
\end{align*}
where the inequality is an application of the AM-GM inequality. Thus, we have shown that with probability at least $1 - \delta$, for all $n$, the second term in \eqref{eqn:sumsq1} is bounded as
\begin{align*}
     \frac{1}{p}\sum_{t=1}^n (p - I_t) (b \wedge \|x_t \|_{V_{t-1}^{-1}}^2 )
     \leq \frac{1}{2}\left(1 \vee \sum_{i=1}^n \left(\|x_i \|^2_{V_{i-1}^{-1}} \wedge b\right) \right)  + Z.
\end{align*}
where $Z = 1.45 \frac{b}{p}  \left( 1.4 \ln \ln \left( 2bn \vee 2\right) + \ln \frac{5.2}{\delta}\right)$.
And when combining all bounds on the sum of squares term in \eqref{eqn:sumsq1}, we get that either $\sum_{i=1}^n \left(\|x_i \|^2_{V_{i-1}^{-1}} \wedge b\right) \leq 1$ or
\begin{align*}
    \sum_{i=1}^n \left(\|x_i \|^2_{V_{i-1}^{-1}} \wedge b\right) &\leq 2Z + \frac{2}{p} \left(1 + b\right) \ln \frac{\det V_n}{\det V_0}\\
    & \leq 
    \frac{4}{p}(1 + b) \ln \frac{\ln(2bn \vee 2) 5.2\det V_n}{\delta \det V_0}
\end{align*}
which gives the desired statement.
\end{proof}

\begin{lemma}[Uniform empirical Bernstein bound]
\label{lem:uniform_emp_bernstein}
In the terminology of \citet{howard2018uniform}, let $S_t = \sum_{i=1}^t Y_i$ be a sub-$\psi_P$ process with parameter $c > 0$ and variance process $W_t$. Then with probability at least $1 - \delta$ for all $t \in \NN$
\begin{align*}
    S_t &\leq  1.44 \sqrt{(W_t \vee m) \left( 1.4 \ln \ln \left(2 \left(\frac{W_t}{m} \vee 1\right)\right) + \ln \frac{5.2}{\delta}\right)}\\
   & \qquad + 0.41 c  \left( 1.4 \ln \ln \left(2 \left(\frac{W_t}{m} \vee 1\right)\right) + \ln \frac{5.2}{\delta}\right)
\end{align*}
where $m > 0$ is arbitrary but fixed.
\end{lemma}
\begin{proof}
Setting $s = 1.4$ and $\eta = 2$ in the polynomial stitched  boundary in Equation~(10) of \citet{howard2018uniform} shows that $u_{c, \delta}(v)$ is a sub-$\psi_G$ boundary for constant $c$ and level $\delta$ where 
\begin{align*}
    u_{c, \delta}(v) &= 1.44 \sqrt{(v \vee 1) \left( 1.4 \ln \ln \left(2 (v \vee 1)\right) + \ln \frac{5.2}{\delta}\right)}\\
   &\quad  + 1.21 c  \left( 1.4 \ln \ln \left( 2 (v \vee 1)\right) + \ln \frac{5.2}{\delta}\right).
\end{align*}
By the boundary conversions in Table~1 in \citet{howard2018uniform} $u_{c/3, \delta}$ is also a sub-$\psi_P$ boundary for constant $c$ and level $\delta$. The desired bound then follows from Theorem~1 by \citet{howard2018uniform}.
\end{proof}